\documentclass[11pt,a4paper]{article}

\usepackage[pdftex,plainpages=false,colorlinks=false,hyperindex,bookmarksopen]{hyperref}

\usepackage[english]{babel}
\usepackage{comment}
\usepackage{amssymb,amsmath,amsfonts,amsthm}

\usepackage{tikz}
\usepackage{graphicx}
\usepackage{listings}
\usepackage{subfigure}
\usepackage{pifont}
\usepackage{tikz-cd}
\usepackage{rotating}
\usepackage{tabularx,colortbl}
\usepackage[a4paper]{geometry}
\usepackage{datetime}
\usepackage{todonotes}
\usepackage{xspace}
\usepackage[defblank]{paralist}
\usepackage{enumerate}
\usepackage[scaled=.85]{newtxtt} 
\usepackage{booktabs}
\usepackage{multirow}

\usepackage[T1]{fontenc}
\usepackage[latin1]{inputenc}
\usepackage{a4wide}

\usepackage{fancybox}
\usepackage{fancyvrb}

\usepackage{graphicx}
\usepackage{latexsym}
\usepackage{paralist}
\usepackage{amsmath}
\usepackage{xparse}
\usepackage{todonotes}
\usepackage[algo2e,lined,ruled,linesnumbered]{algorithm2e} 

\usepackage{filecontents}
\usepackage{pgfplots}

\usepackage{multicol}

\usepackage{lipsum}
\makeatletter
\if@titlepage
{\if@twocolumn\else\endquotation\fi}
\fi
\makeatother

\usepackage{fancyhdr}
\usetikzlibrary{arrows,shapes,snakes,automata,backgrounds,petri,positioning,shadows,matrix,decorations.pathmorphing,fit,calc,backgrounds}

\newtheorem{lemma}{Lemma}[section]
\newtheorem{theorem}{Theorem}[section]

\newtheorem{definition}[theorem]{Definition}

\theoremstyle{remark}

\newtheorem*{example}{Example}

\newcommand{\set}[1]{\{#1\}}
\newcommand{\tup}[1]{\langle #1\rangle} 

\newcommand{\A}{\mathcal{A}}

\newcommand{\I}{\mathcal{I}}

\newcommand{\M}{\mathcal{M}}

\newcommand{\R}{\mathcal{R}}

\newcommand{\goto}[1]{\mathrel{\raisebox{-2pt}{$\xrightarrow{#1}$}}}

\newcommand{\myi}{\emph{(i)}\xspace}
\newcommand{\myii}{\emph{(ii)}\xspace}
\newcommand{\myiii}{\emph{(iii)}\xspace}
\newcommand{\myiv}{\emph{(iv)}\xspace}
\newcommand{\myv}{\emph{(v)}\xspace}

\newcommand{\defterm}[1]{\ul{\textit{\mbox{#1}}}}	

\usepackage{color,soul}

\renewcommand{\defterm}[1]{\textbf{#1}}

\newcommand{\anEformula}{\psi(\uj,\self,e,\uv)}

\newcommand{\nop}{nop\xspace}

\newcommand{\act}{\alpha}
\newcommand{\actvec}{\vec{\act}}
\newcommand{\Acts}{\A}

\newcommand{\ActsA}{\A^{\mathit{loc}}}
\newcommand{\ActsAE}{\A^{\mathit{syn}}}

\newcommand{\async}{interleaved\xspace}
\newcommand{\sync}{concurrent\xspace}

\newcommand{\TOOL}{\textsl{SAFE}\xspace}
\newcommand{\TOOLtitle}{\TOOL: the Swarm Safety Detector\xspace}

\newcommand{\Rels}{\R}
\newcommand{\valueof}{\mathit{val}}
\newcommand{\var}{\mathit{v}}
\newcommand{\Vars}{V}

\newcommand{\dom}{\mathit{type}}
\newcommand{\IDS}{\textit{ID}}

\newcommand{\Types}{\Theta}
\newcommand{\type}{\theta}

\newcommand{\id}{\textsc{id}\xspace}

\newcommand{\goal}{\psi_{\mathit{goal}}}
\newcommand{\goalab}{\overline{\goal}}

\newcommand{\self}{\textit{\small{self}}\xspace}

\DeclareDocumentCommand{\locstateof}{ m o }{%
        \IfNoValueTF{#2}%
            {g{.}{#1}}%
            {#2{.}{#1}}%
}

\newcommand{\ABM}{{\textit{ab}(\M)}}

\newcommand{\localactof}[1]{\actvec{.}{#1}}
\newcommand{\varof}[1]{\var^{[#1]}}

\newcommand{\Eformula}{agent formula\xspace}

\newcommand{\mytexttt}[1]{{\small \texttt{#1}}}

\newcommand{\safe}{\texttt{SAFE}\xspace}
\newcommand{\unsafe}{\texttt{UNSAFE}\xspace}

\newcommand{\PMAS}{PMAS}
\newcommand{\ABMAS}{AB-\PMAS\xspace}
\newcommand{\MASsystem}{\PMAS\xspace}

\newcommand{\arracts}{\underline{act}}

\newcommand{\arr}{\textit{arr}}
\newcommand{\uarr}{\underline{\textit{arr}}}

\newcommand{\absVar}{\textit{X}}
\newcommand{\absActs}{\Acts^{a}}
\newcommand{\globvars}{\underline{glob}}
\newcommand{\absglobvars}{\underline{x}}
\newcommand{\env}{\textit{env}}

\newcommand{\phase}{{\small \textit{Phase}}}

\newcommand{\uj}{\ensuremath{\underline j}}

\newcommand{\ua}{\ensuremath{\underline a}}

\newcommand{\ue}{\ensuremath{\underline e}}
\newcommand{\ui}{\ensuremath{\underline i}}

\newcommand{\uv}{\ensuremath{\underline v}}
\newcommand{\ux}{\ensuremath{\underline x}}

\newcommand{\uz}{\ensuremath{\underline z}}
\newcommand{\uF}{\ensuremath{\underline F}}
\newcommand{\uG}{\ensuremath{\underline G}}

\newcommand{\cC}{\ensuremath \mathcal C}

\newcommand{\cM}{\ensuremath \mathcal M}
\newcommand{\cN}{\ensuremath \mathcal N}
\newcommand{\cS}{\ensuremath \mathcal S}

\begin{document}

\title{SMT-based Safety Verification of\\Parameterised Multi-Agent Systems}

\author{Paolo Felli$^1$, Alessandro Gianola$^{1,2}$, Marco Montali$^1$ \\
\\
 $^1$ Free University of Bozen-Bolzano, Bolzano, Italy\\
 \texttt{\textit{surname}@inf.unibz.it}\\
 $^2$ University of California San Diego, San Diego (CA), USA\\
 \texttt{agianola@eng.ucsd.edu}}

\date{August 10th, 2020}

\maketitle

\begin{abstract}
In this paper we study the verification of parameterised multi-agent systems (MASs), and in particular the task of verifying whether unwanted states, characterised as a given state formula, are reachable in a given MAS, i.e., whether the MAS is unsafe. 
The MAS is parameterised and the model only describes the finite set of possible agent templates, while the actual number of concrete agent instances for each template is unbounded and cannot be foreseen. This makes the state-space infinite. As safety may of course depend on the number of agent instances in the system, the verification result must be correct irrespective of such number. We solve this problem via infinite-state model checking based on satisfiability modulo theories (SMT), relying on the theory of array-based systems: 
we present parameterised MASs as particular array-based systems, under two  execution semantics for the	 MAS, which we call \sync and \async. We prove our decidability results under these assumptions and illustrate our implementation approach, called \TOOLtitle, based on the third-party model checker MCMT, which we evaluate experimentally. Finally, we discuss how this approach lends itself to richer parameterised and data-aware MAS settings beyond the state-of-the-art solutions in the literature, which we leave as future work. 
\end{abstract}

\newpage
\tableofcontents

\newpage


\section{Introduction}
\label{sec:introduction}

Multi-agent systems (MASs) are commonly used in many complex, real-life domains, so it has become crucial to be able to verify such systems against given specifications. This typically amounts to check the existence of execution strategies for the achievement of given goals or to compute counterexamples as evidence of points of potential failure. Model checking \cite{Clarke2018} is one of the most common approaches to verification of MASs, often with a focus on strategic abilities  \cite{BullingGJ15}. However, a common limitation in this literature is the assumption that the system is finite-state and fully specified, which in many applications requires to propositionalize crucial system features. Other approaches have thus tackled the verification of MASs in settings that are intrinsically infinite-state  \cite{EsparzaGLM17}, for which explicit model-checking techniques cannot be used off-the-shelf. These are the settings in which either some sort of \emph{data component} is present or where the concrete component instances of the MAS are not explicitly listed beforehand. 
Our work falls into the latter category, that  is the one of verification of parameterised MASs (\MASsystem{s}), recognised as a key reasoning task and addressed by a growing literature \cite{KouvarosAIJ16,Bloem15,KouvarosAAAI17,EsparzaGLM17,ConduracheMG19}. 
In \MASsystem{s}, the number of agents is \emph{unbounded} and \emph{unknown}, so that possibly infinite concrete MASs need to be considered: the task is to check whether  the specification is met by any (or all) concrete MAS that adheres to some behavioural structure (typically a set of agent \emph{templates}), without fixing the number of actual agents a priori. 
Here, we focus on checking \emph{safety}, namely that no state satisfying a state formula (existentially quantifying on agent instances) is reachable for any number of agent instances. E.g., checking that there will never be two agents in the restricted area. 
Note that this differs from checking that a strategy (for some agent) exists to prevent  unsafe states. Safety checking (and reachability) is a crucial property of MASs as well as finite and infinite dynamic systems, with a long-standing tradition (e.g., \cite{Abdulla96}).
Applications are numeorous, from the verification of properties of swarms to industry 4.0 \cite{AlechinaBGFLV19}, where one wants to check that instances of a product family will be manufacturable by robots from a fixed model catalogue. 

In this paper we present our verification technique based on an SMT \cite{BarrettT18} approach for array-based systems \cite{ijcar08,lmcs,FI,fmsd,CGGMR19,MSCS20}, characterising its soundness and completeness (and decidability of the task) under different assumptions. We detail our solution and comment on its implementation in the well-established SMT-based model-checker MCMT \cite{GhilardiR10}. As future work, we comment on how this framework lends itself to accommodate the other  source of infinity mentioned above, i.e. the \emph{data} dimension. 

The reminder of this paper is organized as follows.  In Section~\ref{sec:related}, we state the contributions of this work and relate our results to the previous literature on parameterised verification and, in particular, verification of parameterised multi-agent systems. Then, in Section~\ref{sec:mas} we provide the definition of Parameterized MAS (PMAS) and we present two different semantics for PMASs, i.e.  the concurrent and the interleaved ones:  this distinction gives rise to two corresponding classes of PMASs. In Section~\ref{sec:runs} the (un)safety checking problem for PMASs is introduced; all the results of the paper will focus on solving this kind of verification task. Then, in Section~\ref{sec:mas-abs} we give a formal encoding of PMASs into the array-based systems formalism: this translation depends on the specific semantics employed, hence two distinct encodings for interleaved and concurrent PMASs are provided. In Section~\ref{sec:results} we describe our symbolic variant of backward search, i.e. the verification procedure we employ to assess (un)safety of PMASs, and we discuss its soundness and completeness. In Section~\ref{sec:termination} we also device an additional (sufficient) syntactic condition to impose to PMASs in order to guarantee termination of backward search, and hence the decidability of the verification problem. Finally, in Section~\ref{sec:implem} we present and evaluate experimentally our user-interface tool, called SAFE, illustrating the implementation of our approach and how it is based on the state-of-the-art MCMT model checker. We conclude and discuss future work in Section~\ref{sec:futurework}.

First, we introduce our running example (a further example will be discussed in Section~\ref{sec:evaluation}).

\bigskip
\textbf{Example scenario.} 
Imagine a robotic swarm attacking a defence position, protected by a robot cannon. There are only two possible paths to reach the position: an attacker must first move to waypoint A or B, then move again to reach the target. Attackers can only move to either waypoint if the paths to A or B are not covered in snow, and similarly for reaching the target. The snow condition is not known in advance. The defensive cannon can target either waypoint with a blast or with an EMP pulse. The cannon program is so that a blast can only be fired if there are robots in that waypoint, and at least one robot under fire is hit.   
If instead the EMP pulse is directed towards A or B, no robot can move there. The cannon can use either the blast or activate the pulse at the same time but, while the EMP is active, the cannon can continue firing blasts. The EMP can cover either A or B, not both. 
The number of attackers is not known. 

It appears that, even if all paths are free of snow, an effective defensive plan exists: at the beginning, use the pulse (say against A); let the enemy robots make their moves to B (the pulse ramains active on A), then use the blast against B to destroy robots there. Whenever further attackers move to B, use again the blast, otherwise wait. 
If one path is not viable the plan is even simpler. 
Question: does this strategy work? Answer: only if the blast destroys all robots in the waypoint against which it is fired. Q: if blasts do not hit all robots under fire, how many attackers may have a chance of reaching the target? A: at least two, if they move to the waypoint B together, since blasts always hit at least one robot. Q: is any attack plan for at least two robots guaranteed to work? A: no. 
This scenario is trivial, however a complex network of waypoints or cannons capable of targeting multiple waypoints can  make this arbitrarily complex, also given that the snow conditions cannot be foreseen, requiring to reason by cases. If ``playing'' as attackers, computing an attack plan (and for how many robots) that has chances against any number of  cannons is even more complex. 

In this paper we tackle this type of scenarios, showing how they can be modelled and solved, but also that our solution technique is powerful enough to be used to account for a number of features that cannot be included in this preliminary work, that is, the inclusion of full fledged relational database storing public and private agent data information with read and write access.

%



\section{Related Work and Contribution}
\label{sec:related}

The related work is constituted by the literature on parameterised verification  \cite{Bloem15} and more specifically verification of PMASs. 

\subsection{Parameterized verification}
The literature on parameterized verification is related but nonetheless distinct from our approach, for al least two reasons. First, we are tackling verification of parameterised \MASsystem{s},  tightly relating our decidability results with the assumed MAS execution semantics and justifying our modeling choices based on that setting. 
Second, as summarized in \cite{Bloem15}, decidability results for these systems are based on reduction to finite-state model checking via abstraction \cite{Pnueli02,john2012counter}, \emph{cutoff} computations (i.e. a bound on the number of instances that need to be verified \cite{Emerson03}) or by proving that they can be represented as well-structured transition systems \cite{Finkel01,Abdulla96}.  
However, our technique is not based on (predicate and counter) abstractions, cutoffs or reductions to finite-state model checking.    
Instead, our theoretical results are based on the model-theoretic framework of ABS \cite{lmcs,MSCS20} and can be seen as a declarative, first-order counterpart of theories of well-structured transition systems for which compatible results are known in the community (see, e.g., \cite{Abdulla96,Bloem15}).  
%
Finally, as we argue in the next section, the application of our results is novel, effective and yields decidability results of direct and immediate applicability to a clear class of \MASsystem{s}. 

\subsection{PMAS verification}
\label{sec:related_pmas}

Regarding verification of \MASsystem{s}, the closest related work is that on parameterised MAS \cite{KouvarosAIJ16,KouvarosAAAI17,BelardinelliIJCAI17} and open MAS  \cite{ConduracheMG19,KouvarosLPP19}.  
In \cite{ConduracheMG19,KouvarosLPP19}, the authors study MASs where agents can join and leave dynamically. As in our work, agents are characterised by a type and their number is not bounded. Types in \cite{ConduracheMG19} are  akin to ours, although their evolution is action-nondeterministic due on observations. 
%
\cite{ConduracheMG19} adopts synchronous composition operations over automata on infinite words 
and their procedure can  verify strategic abilities for LTL goals by reduction to synthesis. 
Notwithstanding the fact that we only look at safety (reachability), a mechanism for joining/leaving the system can be captured natively in our formalization of \MASsystem.  
A similar framework is in \cite{KouvarosLPP19}, sharing the same model of 
\cite{KouvarosAIJ16} and related papers,  
with agent templates similar to those considered here.   
Compared with that work, we restrict to safety checking instead of considering the more general task of model checking specifications in modal logics or strategic abilities. Safety checking (and, conversely, reachability analysis) is a crucial task with a long tradition in AI and in the field of reasoning about actions.  
In what follows, we note some comparison points, highlighting the value and novelty of our approach for checking safety.  

$\bullet$ \textbf{In the theory.}
In this paper we present a  verification technique based on an SMT \cite{BarrettT18} approach for ABS  \cite{ijcar08,lmcs,FI,fmsd,CGGMR19}, characterising its soundness and completeness. 
%
This is a very well-understood SMT-based theory for which a number of results of practical applicability already exist, and research is active \cite{lmcs,CGGMR19,GhilardiR10,CGGMR18,cade19,IJCAR20,IJCAR20-ext,MSCS20,BPM19,BPM20,ARCADE}. 
This is the first paper to establish a formal connection between verification of \MASsystem{s} and the long-standing tradition of SMT-based model checking for ABS. 
%
Also,  leveraging SMT-techniques makes the framework directly extendible in multiple directions. For example, we only used the empty theory and EUF (the theory of uninterpreted symbols): these theories are customary in the SMT literature, and are sufficient for the scope of this paper, since here we deal with first-order relations that are \emph{uninterpreted} (i.e. not interpreted in \emph{any} specific domain). 
Therefore, more involved theories are not used, but this possibility is readily available thanks to our work: indeed, our formalism can be immediately adapted to the case of richer theories that can be employed to impose additional constraints on the value domains of interest. 
We can easily introduce theories constraining agent data. E.g.,  elements can be retrieved from relational databases (both shared or private) with constraints such as key dependencies, in the line with~\cite{CGGMR19,BPM19,MSCS20}. On this, it is our aim to combine this framework with the RAS formalism in~\cite{CGGMR19,MSCS20}, in order to equip a \MASsystem with relational data read/written by agents.  
Adding theories, data-aware extensions, restricted arithmetics, cardinality constraints, are all now concretely usable directions for  checking safety of \MASsystem{s}. 

$\bullet$ \textbf{In the execution semantics.}
In \cite{KouvarosAIJ16}, authors consider \myi asynchronous actions and synchronization  between the environment and \myii  one agent, \myiii all agents of a given template, \myiv all agents, \myv two agents of different templates. 
 All these are possible in our framework \emph{as it is}, although we explicitly describe only \myi, \myii, \myiii.  Further execution semantics can be trivially reconstructed, based on these.

$\bullet$ \textbf{In the decidability guarantees.} 
Even though our objective \emph{is not} that of reconstructing known frameworks,  it is useful to note that our results are compatible but not easily comparable with the literature on \MASsystem{s} verification. 
The results in \cite{KouvarosAIJ16} depend on the combinations of actions as above: it is decidable only for \MASsystem where synchronization actions as in \myi, \myiii, \myiv are used (called SFE), while decidability depends in different ways on the existence of a cutoff 
for \MASsystem{s} with either \myi, \myii, \myv (called SMR) or \myi, \myii, \myiv (called SGS). 
The work in \cite{KouvarosAAAI17} extends the cutoff results to the infinite state templates for SGS  (while, to the best of our knowledge, the same extension is not available for SMR). 
As our theory is not restricted in any way to finite dimensions, our results extend to that setting as well (see point ``In the data dimension" below). For SMR and SGS their procedure requires to check the existence of a cutoff; if it exists, the outcome is correct, otherwise the procedure halts with no result. However, the existence of a cutoff depends on the existence of a simulation property (between the agent templates and the environment) to be checked on the abstracted system, which has to be computed first. 
This implies (see Def.~\ref{def:properties}), that in the general case their approach is sound, not complete and terminates, 
and is also complete for a given class of \MASsystem{s} if the existence of cutoffs is guaranteed for that class (such as SFE, but not the whole SMR and SGS).   

Conversely, our technique does not require cutoffs nor any notion alike: we can directly prove soundness and completeness 
for SMR and SGS 
(see Thm.~\ref{thm:sound-complete}), while termination  (thus a complete decision procedure) can be directly guaranteed by a \emph{syntactic} property of the action guards and of the goal formula, which we call \emph{locality} in Section~\ref{sec:termination}. 
%
%
Hence for SMR they cannot provide guarantees without a simulation test or when its result is false, while we can directly characterise the class of problems for which ours is a decision procedure. 
%

%

%
%
%
%

$\bullet$ \textbf{In the data dimension.}
Although we assume finite local states and actions (as in \cite{KouvarosAIJ16} but not in \cite{BelardinelliIJCAI17}), our aim is a clear separation between the sources of infiniteness: 
our theory allows for a data-aware extension for which solid results in the verification literature on ABS exist \cite{CGGMR19,MSCS20}.  
This allows finite action signatures with infinite number of possible parameter values, and also to write infinite data values on a read/write data storage. 
This will give agents infinite possible contexts and action instances,  in a framework natively supporting a data dimension.  
%
%

$\bullet$ \textbf{In the implementation.} 
Finally, a \emph{general purpose} model checker (MCMT \cite{GhilardiR10}) is available for checking safety of \MASsystem{s}: our contribution is directly operational and only requires a textual counterpart of the encoding we propose in Sec.~\ref{sec:encoding_async} (see Sec.~\ref{sec:implem}). Such textual representation, representing a \MASsystem as an array-based system, can be directly given as input file to MCMT. Moreover, since both the modelling of \MASsystem{s} and their translations as textual files for MCMT proves to be particularly laborious, we have made available an intuitive web tool that serves as user interface for modelling and translating \MASsystem{s} into array-based systems, using the syntax required by MCMT. Such tool, called \TOOLtitle, is illustrated in Section~\ref{sec:tool} and is available at \cite{SAFE}.



\section{\MASsystem{s}: parameterised MAS}
\label{sec:mas}

In this section we introduce our representation of \MASsystem{s} and two alternative execution semantics. 

We consider a set $\Types$ of (semantic) data types, used for variables. Each type $\type \in \Types$ comes with a (possibly infinite) domain $\Delta_{\type}$, 
and a type-wise equality operator $=_\type$. For instance, types are reals, integers, booleans, etc.
We simply write $=$ when the type is clear. 
We also consider a set $\Rels$ of relations  over types in $\Types$, which we treat as \emph{uninterpreted} relations (i.e. simple relation symbols). 
These relations are used to model background information in the MAS but \emph{are never updated during its execution}, constituting a \emph{read-only} component. E.g., the snow condition in the scenario can be modelled via these relations, as we will show. 
We consider the usual notion of FO interpretations $\I = (\Delta^\I,\cdot^\I)$ with $\Delta^\I = \bigcup_{\Types} \Delta_\type$ and $\cdot^\I$ is an FOL interpretation function for symbols in $\Rels$.

%
%


\begin{definition}
An \defterm{agent template} is a tuple $T = \tup{\IDS,\allowbreak  L,\allowbreak l^0,\allowbreak \Vars,\allowbreak \dom,\allowbreak \valueof,\allowbreak \ActsA,\allowbreak \ActsAE,\allowbreak P,\allowbreak \delta}$ composed of:   
\begin{compactitem}
\item an infinite set $\IDS$ of unique agent identifiers of sort \id;
\item a finite set $L$ of local states, with initial state $l^0\in L$;
\item a finite set $\Vars$ of local (i.e., internal) agent state variables;
\item a variable-type assignment $\dom : \Vars \mapsto \Types$;
\item a variable assignment $\valueof : L \times \Vars \mapsto \bigcup_{\Types} \Delta_\type$ returning their current values, with $\valueof(l,\var)\in \Delta_\type$ for $\type=\dom(\var)$; 
\item a non-empty, finite set of action symbols $\Acts \doteq \ActsA \cup \ActsAE$ (described later), with $\ActsA \cap \ActsAE = \emptyset$; 
\item a protocol function specifying the conditions under which each action is \emph{executable}. It is a function $P :\Acts \mapsto \Psi$, where $\Psi$ are \Eformula{e}, defined in the next section, that allow to ``query'' the current state of the whole MAS; 
\item a transition function $\delta : L\times \Acts \mapsto L$, describing how the local state is affected by the execution of an action $\act$:  the template moves from a state $l$ to a state $l'$ iff $\delta(l,\act)=l'$, also denoted $l\goto{\act}l'$. We assume $\delta$ to be total.
\end{compactitem}
\label{def:agent_template}
\end{definition}

An \defterm{environment template} is a special agent template $T_e$ with 
fixed identifier (i.e., $\IDS=\set{e}$): there is exactly one environment. 
%
%
Intuitively, a \defterm{(concrete) agent} is a triple composed of an agent $\id$, its template and its current local state. Analogously, a \defterm{(concrete) environment} is a pair consisting of the template $T_e$ and its current state (again, $e$ is a constant). 

\paragraph{Remark.} Please note that the templates are defined as in Definition~\ref{def:agent_template} for consistency with relevant literature and for the sake of clarity, although such an explicit finite-state representation, although clearly possible, is often impractical for real-world implementations. An equivalent representation, which however would require a more tedious formalization, is so that the transition function $\delta$ is given implicitly, rather than explicitly, by specifying \emph{pre- and post-conditions} of actions, based on the current variable assignment. This would allow to specify that an action is executable if and only if a boolean formula  on local variables, mentioning positive and negated variable conditions, is currently satisfied. Accordingly, local states would not be explicit: the set $L$ would be implicitly defined as a finite subset of possible assignments of variables. This is the representation we use in our implementation. As said, however, for the technical development that follows, the representation used in Definition~\ref{def:agent_template} is more convenient. 

\begin{example} 
We use a template $T_{\mathit{att}}$  for robots, with variables \mytexttt{loc} (enumeration [\mytexttt{init,A, B,target}]) and \mytexttt{destroyed} (boolean). The first variable is used for storing an agent's location, whereas the latter is for specifying whether the agent is destroyed and cannot move anymore. 
The actions are \mytexttt{gotoA}, \mytexttt{gotoB} and \mytexttt{gotoT} for representing the actions of moving to waypoints (from the initial location) and to the target (from either waypoint), plus additional actions $\mytexttt{blastA}$, $\mytexttt{blastB}$ representing the action of ``being destroyed'' by a shot fired at position $\mytexttt{A}$ or $\mytexttt{B}$, respectively. For instance, a transition $l\goto{\mytexttt{gotoA}} l'$ exists in $\delta$ for this template when $\valueof(l,\mytexttt{loc})=\mytexttt{init}$ and $\valueof(l,\mytexttt{destroyed})=\mytexttt{false}$, and the resulting local state $l'$ is such that $\valueof(l',\mytexttt{loc})=\mytexttt{A}$ (plus further assignments for inertia). Other actions are defined in a similar manner. Figure~\ref{fig:simple_example} depicts this template, represented as a labelled finite-state system. 
The cannon is modeled as (part of) the environment, whose template $T_e$ has actions \mytexttt{pulseA},\mytexttt{pulseB},\mytexttt{blastA},\mytexttt{blastB} and variables \mytexttt{pulse-loc} (enumeration [\mytexttt{A}, \mytexttt{B}, \mytexttt{nil}]). The former is used to store the location (waypoint  \mytexttt{A} or  \mytexttt{B}) towards which the pulse is currently directed.  

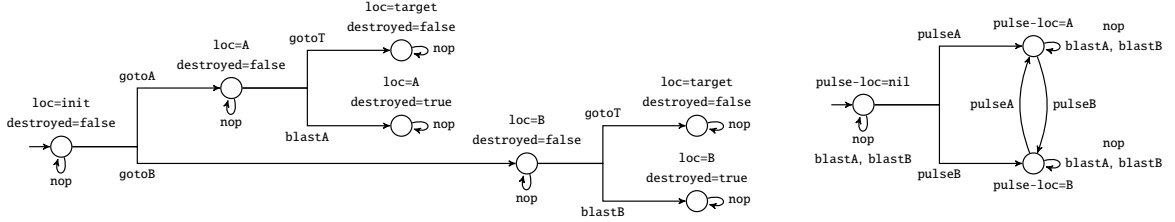
\begin{figure}[t]
\centering{
\resizebox{\columnwidth}{!}{

\begin{tikzpicture}[>=stealth', [ ->, >=stealth', auto, thick, font=\sffamily, node distance=3.5cm ]

\tikzset{place/.style={initial text=,draw=black,circle,font=\small,minimum size=5mm,inner sep=1pt}}

\begin{scope}[shift={(0cm,0cm)}]

\node[place,  initial left, label={[align=center]90:{\texttt{loc=init} \\ \texttt{destroyed=false}}}] (0) at (0,0) {};
\node[place, right=of 0, yshift=1.4cm, label={[align=center]90:{\texttt{loc=A} \\ \texttt{destroyed=false}}}] (1) {};
\node[place, right=of 0, shift={(7cm,-.4cm)}, label={[align=center]90:{\texttt{loc=B} \\ \texttt{destroyed=false}}}] (2) {};
\node[place, right=of 1, yshift=.9cm, label={[align=center]90:{\texttt{loc=target} \\ \texttt{destroyed=false}}}] (3) {};
\node[place, right=of 1, yshift=-.9cm, label={[align=center]90:{\texttt{loc=A} \\ \texttt{destroyed=true}}}] (4) {};
\node[place, right=of 2, yshift=.9cm, label={[align=center]90:{\texttt{loc=target} \\ \texttt{destroyed=false}}}] (5) {};
\node[place, right=of 2, yshift=-.9cm, label={[align=center]90:{\texttt{loc=B} \\ \texttt{destroyed=true}}}] (6) {};

\draw[->] (0) -- ++(1.8,0) |- node[] {\texttt{gotoA}}  (1);
\draw[->] (0) -- ++(1.8,0) |- node[below] {\texttt{gotoB}}  (2);
\draw[->] (1) -- ++(1.8,0) |- node[] {\texttt{gotoT}}  (3);
\draw[->] (1) -- ++(1.8,0) |- node[below] {\texttt{blastA}}  (4);
\draw[->] (2) -- ++(1.8,0) |- node[] {\texttt{gotoT}}  (5);
\draw[->] (2) -- ++(1.8,0) |- node[below] {\texttt{blastB}}  (6);

\path[->] (0) edge[loop below] node[] {\texttt{nop}} (0);
\path[->] (1) edge[loop below] node[] {\texttt{nop}} (1);
\path[->] (2) edge[loop below] node[] {\texttt{nop}} (2);
\path[->] (3) edge[loop right] node[] {\texttt{nop}} (3);
\path[->] (4) edge[loop right] node[] {\texttt{nop}} (3);
\path[->] (5) edge[loop right] node[] {\texttt{nop}} (6);
\path[->] (6) edge[loop right] node[] {\texttt{nop}} (5);

\end{scope}

\begin{scope}[shift={(19cm,1cm)}]

\node[place,  initial left, label={[align=center]90:{\texttt{pulse-loc=nil}}}] (0) at (0,0) {};
\node[place, right=of 0, yshift=1.4cm, label={[align=center]90:{\texttt{pulse-loc=A}}}] (1) {};
\node[place, right=of 0, yshift=-1.4cm, label={[align=center]270:{\texttt{pulse-loc=B}}}] (2) {};

\draw[->] (0) -- ++(1.8,0) |- node[] {\texttt{pulseA}}  (1);
\draw[->] (0) -- ++(1.8,0) |- node[below] {\texttt{pulseB}}  (2);

\path[->] (0) edge[loop below] node[align=center] {\texttt{nop} \\ \texttt{blastA}, \texttt{blastB}} (0);
\path[->] (1) edge[loop right] node[align=center, yshift=.2cm] {\texttt{nop} \\ \texttt{blastA}, \texttt{blastB}} (1);
\path[->] (2) edge[loop right] node[align=center, yshift=.2cm] {\texttt{nop} \\ \texttt{blastA}, \texttt{blastB}} (2);

\path[->] (2) edge[bend left=20] node[] {\texttt{pulseA}} (1);
\path[->] (1) edge[bend left=20] node[] {\texttt{pulseB}} (2);

\end{scope}

\end{tikzpicture}
}}
\caption{A depiction of the templates $T_{\mathit{att}}$ (left) and $T_e$ (right). The label next to each local state $l$ (represented as a state) specifies the value $\valueof(l,\var)$ of each variable $\var\in \Vars$. More transitions could be trivially included in $T_{\mathit{att}}$, but we keep this example minimal.}
\label{fig:simple_example}
\end{figure}

The snow  is captured by a binary relation on locations (e.g.  $Snow(\mytexttt{init,A})$). 
Note that relations in $\Rels$ are common to the whole MAS, and in general their interpretation is unbounded. 
Protocols are given later. 
\qed  
\end{example}

Let $\set{T_1,\ldots,T_n, T_e}$ be a set of agent (and environment) templates, with $T_t =  \tup{\IDS_t,  L_t, l^0_t,\allowbreak \Vars_t,\allowbreak \dom_t,\allowbreak \valueof_t,\allowbreak \ActsA_t,\allowbreak \ActsAE, P_t, \delta_t}$ for  $t\in\set{1,\ldots,n,e}$. We denote a concrete agent of type $T_t$, $t\in[1,n]$, and $\id$ $j$ by writing $\tup{j,l_j}_t$, and similarly we denote the concrete environment by $\tup{e,l_e}_e$. We also denote a vector of $k$ such concrete agents of type $T_t$ as $\tup{\vec{I},\vec{L}}_{t}$, where $\vec{I}\in \IDS_t^k$ and $\vec{L}\in L_t^k$ are vectors of \id{s} and local states, respectively. 
Importantly, we assume that agent \id{s} are unique 
and template variables disjoint, i.e., $\IDS_t \cap \IDS_{t'} = \emptyset$ and $\Vars_t \cap \Vars_{t'} = \emptyset$ for  $t,t'\in\set{1,\ldots,n,e}$, $t\neq t'$. 

A \defterm{\MASsystem} is a tuple $\M=\tup{\set{T_1,\ldots,T_n}, T_e, \Rels}$ with a set of $n$ agent templates, one environment template and the relations. Note that a \MASsystem specifies the initial local state of all agents for each template, but \emph{does not specify how many concrete agents exist for each template}. 
A \defterm{snapshot} 
is  a tuple $g=\tup{\set{\tup{\vec{I}_1,\vec{L}_1}_1,\ldots,\tup{\vec{I}_n,\vec{L}_n}_n},\tup{e,l_e}_e}$, which thus identifies the  number of agent instances (the size of each $\vec{I}_t$, $t\in[1,n]$, may differ). A snapshot is \emph{initial} iff all  agents are in their initial local state. Clearly, \emph{infinite possible initial snapshots exist}, since the number of concrete agents for each template is unbounded and not known a priori. 
As shorthand, we denote the local state $l_j$ of agent $\tup{j,l_j}_t$ in snapshot $g$ as $\locstateof{j}$, thus writing $\tup{j,\locstateof{j}}_t$.
%

\subsection{Agent formulae}

Here we define the \Eformula{e} used for protocols in Def.~\ref{def:agent_template} as   
%
quantifier-free formulae $\anEformula$ where $\uj$ are the free variables of sort \id, \self is a special constant used to denote the current agent, $e$ is the special \id (constant) of the concrete environment, $\uv$ are template variables (for any template). These follow the grammar: 
%
\[
\psi \doteq (\varof{j} = k) ~|~ R(\varof{j_1}_1 , \cdots,\varof{j_m}_m ) ~|~ j_1 = j_2 ~|~ \neg \psi ~|~ \psi_1 \lor \psi_2
\]
where $\var\in\Vars$, $k$ is a constant in $\Delta_\type$ for $\type=\dom(\var)$, $R$ is a relation symbol in $\Rels$ of arity $m\geq 1$ and $j, j_1,\ldots,j_m$ are either the variables of sort \id that appear in $\uj$ or the constant \self or the \id constant $e$ for the environment. The usual logical abbreviations apply. 
Intuitively, these formulas are implicitly quantified existentially over agent \id{s}. 
As we will formalize next, they allow to test (dis)equality of agent variables with respect to agent constants (\id{s}), and to check whether a tuple is in a relation (whose elements are agent variables). For instance, $(\varof{j} = k)$ informally means that there exists an agent \id $j$ so that $\var=k$ for the such agent.  
%
%
%

An \defterm{\id grounding} of a formula $\anEformula$ in $g$ is an assignment $\sigma$ which assigns each variable $j$ of sort \id in $\uj$, and the constants \self and $e$, to a concrete agent \id in $g$, denoted $\sigma(j)$ and $\sigma(\self)$. Specifically, we impose  $\sigma(e)=e$. 
Intuitively, for a formula to be true in $g$, one needs to find a suitable $\sigma$.  

\begin{definition} Given an interpretation $\I_0$, a snapshot $g$ satisfies a formula $\psi$ under $\I_0$, denoted $g\models_{\I_0}\psi$, iff there exists an \id grounding $\sigma$ of $\psi$ in $g$ such that $g,\sigma \models_{\I_0} \psi$, with:
\begin{compactitem}
\item $g,\sigma\models_{\I_0} (\varof{j} = k)$ iff $\valueof_t(\locstateof{\sigma(j)},\var)=k$, where $\var\in \Vars_t$; i.e. the concrete agent $\tup{\sigma(j),\locstateof{\sigma(j)}}_t$ is so that $\var=k$;
\item $g,\sigma\models_{\I_0} R(\varof{j_1}_1 , \cdots,\varof{j_m}_m )$ iff for \id{s} $\sigma(j_1),\ldots,\sigma(j_m)$:
$$R^{\I_0}(\valueof_{t_1}(\locstateof{\sigma(j_1)},\var_1),\allowbreak \cdots,\allowbreak \valueof_{t_m}(\locstateof{\sigma(j_m)},\var_m))$$
where  for each $i\in[1,m]$ we have $\var_i \in \Vars_{t_i}$ for some template $t_i\in\set{1,\ldots,n,e}$. Namely $R$ holds under $\I_0$ for the values of variables  $\var_1 , \ldots,\var_m$ in the local state of agents with \id $\sigma(j_1),\ldots,\sigma(j_m)$ as specified by $\sigma$;
\item $g,\sigma\models_{\I_0} (j_1 = j_2)$ iff $\sigma(j_1)=\sigma(j_2)$;
\item $g,\sigma\models_{\I_0} \neg \psi$ iff $g,\sigma \not \models_{\I_0} \psi$;
\item $g,\sigma\models_{\I_0} \psi_1 \lor \psi_2$ iff $g,\sigma\models_{\I_0}\psi_1$ or $g,\sigma\models_{\I_0}\psi_2$.
\end{compactitem}
%
\end{definition}

Note that \self is freely assigned to an agent \id: if $g$ satisfies a formula with \self, then an agent exists that can be taken as \self. 
%
%
Hence we write $g\models^j_{\I_0}\psi$, if needed, to denote that there exists $\sigma$ with $\sigma(\self)=j$ so that $g,\sigma \models_{\I_0}\psi$. 
This informally reads as \emph{$\psi$ is true in $g$ for agent with \id $j$}.  
E.g., assuming $g$ is s.t. $\var_1=6$ for agent with \id $3$, and $\var_1=5$ for agent with \id $7$, then $g\models^{3}_{\I_0} (\var_1^{[\self]}=6) \land (\var_1^{[j]}=5)$.


\begin{example}\emph{(cont.d)} In the running example, the program of the cannon is so that the cannon can fire on a waypoints only if there is at least on attacker (not already destroyed) in that location. Hence, we have $P_{e}(\mytexttt{blastA}) = (\mytexttt{loc}^{[j]}=\mytexttt{A}) \land (\mytexttt{destroyed}^{[j]}=\mytexttt{false})$ (recall that $\Vars_t\cap\Vars_{t'}=\emptyset$ for $t\neq t'$, hence given a variable we know to which template it belongs). Moreover, only positions that are clear of snow can be accessed: $P_\mathit{att}(\mytexttt{gotoA}) = (\mytexttt{pulse-loc}^{[e]}\neq\mytexttt{A}) \land \neg Snow(\mytexttt{init},\mytexttt{A})$, which imposes that the there is no pulse on the target location and the path to \mytexttt{A} is clear of snow. Similarly we model \mytexttt{blastB}. 
Note that we are assuming that the interpretation $\I_0$ is fixed, hence the snow condition (the relation $Snow$) is fixed and is not affected by the execution. As we are going to show in Section~\ref{sec:runs}, however, we define our safety check so as to check that a condition \emph{cannot} be reached by quantifying universally over all possible initial interpretations, so as to consider any possible snow condition. Conversely, we can check that a given condition \emph{can} in fact be reached for when assuming at least one possible initial interpretation $\I_0$. 
\end{example}

\subsection{Concurrent and \async \MASsystem{s}}

In this section we introduce the two main execution semantics, hence defining two distinct types of \MASsystem{s}, called concurrent and \async. These are distinguished by how the (single-step) transitions of the system are defined, which has to do with the types of interactions that are allowed between the agents and the environment. 

A \defterm{(global) transition} of a  \MASsystem describes the evolution of the \MASsystem when a vector of actions $\vec{\act}$ (one for each concrete agent in $g$ and one for the environment) are executed from a snapshot $g=\tup{\set{\tup{\vec{I}_1,\vec{L}_1}_1,\ldots,\tup{\vec{I}_n,\vec{L}_n}_n},\tup{e,l_e}}$, so that a new snapshot of the form $g'=\tup{\set{\tup{\vec{I}_1,\vec{L}'_1}_1,\ldots,\allowbreak \tup{\vec{I}_n,\vec{L}'_n}_n},\tup{e,l'_e}}$ is reached. This is denoted by simply writing $g\goto{\vec{\act}}g'$.

Since each concrete agent and the environment may perform an action (in $\ActsA_t\cup\ActsAE$) or remain idle, multiple executions semantics can be defined, depending on the constraints we wish to impose on the vector $\vec{\act}$. 
Let us then describe more in detail the sets $\ActsA_t$ and $\ActsAE$ introduced in Definition~\ref{def:agent_template}. 

Symbols in $\ActsA_t$, for each $t$, are called \defterm{local actions}, and those in $\ActsAE$ \defterm{synchronisation} actions.  Actions in  $\ActsA_t$ can only affect the local state of the concrete agent which executes them, whereas actions in $\ActsAE$ represent the synchronisation between one or more agents and the environment and thus can affect the local state of each agent involved. Intuitively, the synchronization actions are used to model explicit communication actions or any action with effects that are not private to the single agent or to the environment. 

As a consequence, not every vector $\vec{\act}$ is meaningful: synchronization actions in $\ActsAE$ are shared across all templates and are used to model global events that are (potentially) observable by any agent, whereas local actions in $\ActsA$ are private and can be freely executed. Typically, one wants to constrain the possible evolutions so that synchronization actions and local actions do not happen at the same time, so that we can distinguish those steps in which the \MASsystem evolves in response to public actions, events or messages from those steps in which agents update their local state in isolation.


\subsubsection{Concurrent \MASsystem{s}}
\label{sec:sync}

First, we focus on those \MASsystem{s} in which whenever a synchronization action is executed then all agents and the environment are forced to synchronize, as formalized below. Intuitively, synchronisation actions are seen as public events, affecting the local state of each concrete agent. We call the resulting class of systems  \defterm{\sync \MASsystem}. 
For capturing this execution semantics,  we adopt the following definition to characterise the global transitions that are said to be \emph{legal}.

\begin{definition}
Given $\I_0$, $g\goto{\vec{\act}}g'$ is \defterm{legal} iff: 
\begin{compactitem}
\item $\locstateof{j}[g']=\delta_t(\locstateof{j}[g],\localactof{j})$ for every $\tup{j,\locstateof{j}}_t$, $t\in\set{1,\ldots,n,e}$, i.e.,  agents and environment evolve as per their template;
\item $g\models^j_{\I_0} P_t(\localactof{j})$ for every $\tup{j,\locstateof{j}}_t$, i.e., each action is executable and  \self is substituted by $j$ when  evaluating the protocol; 
\item  either 
only local actions are executed (by agents and environment), or the environment and at least one agent synchronize with action $\act\in\ActsAE$. In the former case, agents remain idle iff there is no local executable action, and in the latter iff they cannot execute $\alpha$. Formally, either:
\begin{compactitem}
\item 
no $j$ exists so that $\localactof{j}\in \ActsAE$, and for every $\tup{j,\locstateof{j}}_t$ if $\localactof{j}=\nop$ then no $\act\in\ActsA_t$ exists with $g\models^j_{\I_0} P_t(\act)$; 
\item $\localactof{e}=\act \in \ActsAE$ and at least one   $j\neq e$ exists so that $\localactof{j}=\act$. Moreover, for every agent $\tup{j,\locstateof{j}}_t$ either \myi $g\models^j_{\I_0} P_t(\act)$ and $\localactof{j}=\act$ or \myii $g\not \models^j_{\I_0} P_t(\act)$ and $\localactof{j}=\nop$.  
\end{compactitem}
\end{compactitem}
\label{def:legal_global_transition}
\end{definition}

\noindent
Note that local and synchronisation actions cannot be mixed.

\subsubsection{Interleaved \MASsystem{s}}
\label{sec:async}

Next, we formalize \defterm{\async \MASsystem{s}}. In these systems, at each step either 
\begin{inparaenum}[\it (i)]
\item a subset of concrete agents (and the environment) perform a (non $\nop$) action in $\ActsA_t$ on their local state or  
\item the environment and \emph{a subset} of the agents synchronize by executing the same action in $\ActsAE$.  
\end{inparaenum} 
$\nop$ is a special no-op action: $\delta_t(l,\nop)=l$ for all $t,l$. 
Local and synchronization actions cannot be mixed. 
We denote by $\localactof{j}$ the action of the agent with \id $j$, or of the environment if $j=e$.

\begin{definition}
Given an interpretation $\I_0$, $g\goto{\vec{\act}}g'$ is \defterm{legal} iff: 
\begin{compactitem}
\item $\locstateof{j}[g']=\delta_t(\locstateof{j}[g],\localactof{j})$ for every $\tup{j,\locstateof{j}}_t$, $t\in\set{1,\ldots,n,e}$, i.e.,  agents and environment evolve as per their template;
\item $g\models^j_{\I_0} P_t(\localactof{j})$ for every $\tup{j,\locstateof{j}}_t$, i.e., each action is executable and  \self is replaced by $j$ for evaluating protocols; 
\item  either 
only local actions are executed (by some agents and environment), or the environment and at least one agent synchronize  with action $\act\in\ActsAE$. All other agents perform \nop. Formally, either:

\begin{compactitem}[$-$]
\item  
no $j$ exists so that $\localactof{j}\in \ActsAE$, that is, no synchronization action is executed; or 
\item the environment and at least one agent synchronize, while other agents can either synchronize as well or freely decide to remain idle. Formally, $\localactof{e}=\act\in\ActsAE$ and exists  $j\neq e$ with $\localactof{j}=\act$, and $\localactof{j}\in \set{\act,\nop}$ and $g\models^j_{\I_0} P_t(\localactof{j})$ for every  $\tup{j,\locstateof{j}}_t$. 
\end{compactitem}
\end{compactitem}
\label{def:legal_global_transition_2}
\end{definition}

\begin{example} 
In previous examples we did not comment on an important point: the actions 
\mytexttt{blastA/B} are synchronization actions (modeling the firing action and the `being hit' action of robots). 
According to the definition above, when the example is modelled as an \async \MASsystem, then a blasts is not guaranteed to destroy all targets because not all agents in location  \mytexttt{A} are forced to synchronize with such action. In fact, the two cases in which a blast destroys all agents in the location, or just a subset, are elegantly captured by simply assuming a \sync or \async semantics. 
In the former, a global transition including \mytexttt{blastA} (by $e$) is legal iff all agents execute \mytexttt{blastA} as well (recall that this encodes the ``being hit'' effects for those agents in location A, and no-effects for others). 
\label{ex:sync}
\end{example}

\subsubsection{Runs of \MASsystem{s} and the Reachability Task}
\label{sec:runs}

Based on the one-step definition of (legal) global transition, we now define the notion of runs for \sync and \async \MASsystem{s}. 
Given a \MASsystem $\M=\tup{\set{T_1,\ldots,T_n},T_e, \Rels}$, a \defterm{(global) run} is a pair $\tup{\rho,\I_0}$ where $\rho$ is a sequence $\rho=g^0\goto{\vec{\act}^1}g^1\goto{\vec{\act}^2}\cdots$ and $\I_0$ is an interpretation  for relation symbols as before.   
We restrict to runs that \myi  are legal and \myii start from an initial snapshot, i.e., with all concrete agents in their initial local state.  
A global transition as above specifies how each concrete agent $\locstateof{j}$ evolves depending on the nature of the action $\localactof{j}$. 
As stated before, once fixed at the start of $\rho$, $\I_0$ \emph{does not change} and is used at each step for evaluating formulae. 

\begin{definition}
An \Eformula  $\goal$ is \defterm{reachable} in $\M$ iff there exists an initial snapshot $g^0$ of $\M$ s.t. a snapshot $g$ with $g\models_{\I_0}\goal$  is reachable through a run $\tup{\rho,\I_0}$ from $g^0$.
\label{def:reachable}
\end{definition}

The verification task at hand is to assess whether $\goal$ is reachable, i.e. $\M$ is \defterm{unsafe} w.r.t. $\goal$. If a formula is unreachable then it is so for any number of agents and all possible interpretations. In such a cases, $\M$ is said to be \defterm{safe} w.r.t. $\goal$.


\section{\MASsystem as Array-based Systems}
\label{sec:mas-abs}

``Array-based Systems'' (ABS) \cite{lmcs} is a generic term used to refer to
\emph{infinite state transition systems} implicitly specified using a declarative,
logic-based formalism  in which arrays are manipulated 
via logical formulae. 
Intuitively, they describe a system that, starting from an initial configuration (specified by an  initial formula), is progressed through transitions (specified by transition formulae). 
The precise definition depends on the specific application. 
They are described using a multi-sorted
theory with one kind of sorts for the indexes of
arrays and  another for the elements stored therein. The content of an array is unbounded and updated during the evolution. 
Since the content of an array changes over time, it is referred to by a second-order function variable, whose interpretation in a state is that of a total function mapping indexes to elements (so that applying the function to an index denotes the classical \emph{read} array operation). 
We adopt the usual FO syntactic notions of signature, term, atom, (ground) formula, etc.

The definition, for each \emph{array variable} $a$, requires a formula $\iota(a)$ describing its \emph{initial configuration}, and a formula $\tau(a,a')$ describing a \emph{transition} that transforms the content of the array from $a$ to $a'$. Verifying whether the system can reach configurations described by a formula $Z(a)$ amounts to checking whether $\iota(a_0)\wedge \tau(a_0, a_1) \wedge \cdots \wedge \tau(a_{m-1}, a_m)\wedge Z(a_m)$ is satisfiable for some $m$. In the literature, states described by $Z(a)$ are called \emph{unsafe} (typically capturing an unwanted condition), so we preserve this terminology.   

In order to introduce verification problems in the symbolic setting of
ABS, one first has to specify the FO theories $T_{Ind}$ and $T$ (equipped with FO signatures $\Sigma_{Ind}$ and $\Sigma$, resp.) for \emph{array indexes} and for the  \emph{array elements}. In this paper, $T_{Ind}$ will be the empty theory where $\Sigma_{Ind}$ only contains  equality, and $T$ will be EUF (the theory of uninterpreted symbols), i.e., the empty theory with signature $\Sigma$ containing constants and relation symbols. 
This is a standard, common setting in the ABS literature and in the SMT community.  
%
%

Secondly, we have to describe the formulae used to
represent three main components: 
\begin{inparaenum}[\it (i)]
\item the sets of states of the system; 
\item its initialization, describing (the set of) initial states; 
\item the system evolution, capturing state transformations.
\end{inparaenum}

We denote by $\uz$ a tuple
$\tup{z_1,\ldots,z_m}$ and by 
%
%
$\phi(\ux,\ua)$ 
the formula with $\ux$ as free individual variables and $\ua$ as free array variables. 
%
Let $\Sigma:=\tup{\cS,Rel, C}$ be a multi-sorted FO signature with sorts $\cS$, relation symbols $Rel$ and constants $C$, and let $\absVar$ be a set of variables. 
In the following we use the notation 
$ F(\ux) ~:=~ \mathtt{case~of}~ \{\kappa_1(\ux):t_1;\cdots \kappa_n(\ux):t_n\}$ (where $\kappa_i(\ux)$ are quantifier-free $\Sigma$-formulae and $t_i$ are generic terms), or, equivalently, nested if-then-else expressions: we call one such $F$ \emph{case-defined function}. We also use $\lambda$-abstractions  like $b = \lambda j. F(j,\uz)$ in place of $\forall j.~b(j)=F(j,\uz)$, where typically $F$ is a case-defined function or a constant assignment. 
%
Intuitively, this allows to specify assignments via case-defined functions (if-then-else) or "bulk" assignments of all cells of an array. %
%
Hence, we consider three types of formulae:

\noindent
$\bullet$ 
An \emph{initial formula} $\iota(\ux,\ua)$ 
initialises individual and array variables via assignments and $\lambda$-abstractions: 
$ 
  \textstyle
  (\bigwedge_{i=1}^m x_i= c_i) \land
  (\bigwedge_{i=1}^k a_i =\lambda j. d_i),
$ 
with $c_i$,$d_i$ constants from $C$ in $\Sigma$;  for simplicity we assume an \textit{undef} constant by which we can make the initial state fully specified (as for a \MASsystem), but this can be relaxed;

\noindent
$\bullet$ 
A \emph{state formula} of the form  
$
  \exists \uj\, \phi(\uj, \ux,\ua)
$ 
specifies conditions on variables, 
where $\phi$ is a quantifier-free $\Sigma$-formula and  $\uj$ are individual variables of
the index sort;

\noindent
$\bullet$ 
A \emph{transition formula} $\hat\tau$ relates current and new (primed) values of individual and array variables:  
\begin{equation*}\label{eq:trans1}
  \textstyle
  \exists \ue\,(
    \gamma(\ue,\ux,\ua)
    \land (\bigwedge_{i=1}^m x'_i:= c)
    \land (\bigwedge_{i=1}^k a'_i=\lambda j. F_i(j,\ue,\ux,\ua))
    )
\end{equation*}
where $\ue$ are individual variables (of \emph{both} element and index
sorts); $\gamma$ (the `guard') is a quantifier-free $\Sigma$-formula (or a  formula quantifying universally over indexes); 
$\ux'$ and $\ua'$ are renamed
copies of $\ux$ and $\ua$; $c$ is a constant from $C$ in $\Sigma$ and $F_j$ (the ``conditional update'') is a case-defined
function.

We now give a general definition of array-based systems, one that helps us narrow the scope and consider the kind that is suitable for our purposes (e.g. having the notion of action), in place of a generic notion of array-based system, that is extremely general. 
Then we show how a \MASsystem can be encoded as a special case of such definition (in Def.~\ref{def:mas-as-abs}). 
%
Known results on array-based systems (see, e.g., \cite{lmcs,CGGMR19}) can  suitably be adapted to this variant.

\begin{definition}
 An \defterm{abstract \ABMAS}   is a tuple:
$$\tup{\Sigma, \cS_{ind},  \absglobvars, \uarr_{\absVar}, \arracts, \iota_a,  \tau_a}$$
\noindent
\begin{inparaenum}[\it (i)]
\item $\Sigma:=\tup{\cS,Rel, C}$ is a multi-sorted FO signature as above, such that there exists a specific sort $\absActs\in \cS$ called `actions sort' ; 
 \item $\cS_{ind}$ is a set of sorts of index type; 
 \item $\absglobvars$ is a set of individual variables (containing the global variables encoding the states of the environment);
   \item $\uarr_{\absVar}$ is a set of arrays, one for each variable $x\in\absVar$;
   \item $\arracts$ is a set of arrays, with codomain of type 
   $\absActs$; 
  \item $\iota_a$ is an initial formula, whose individual variables are $\absglobvars$ and whose array variables are $\uarr_{\absVar}$ and $\arracts$;
  \item $\tau_a$ is a disjunction of transition formulae, with individual and array variables $\absglobvars, \absglobvars'$ and $\uarr_{\absVar}, \arracts , \uarr_{\absVar}', \arracts'$, resp.
  \end{inparaenum}
  \label{def:abstract-ABMAS}
\end{definition}

Formally, a FO interpretation of $\Sigma$  can be thought as an \emph{instance} of the \emph{`elements' domain} of an abstract \ABMAS, the individual variables are assigned to values taken from this interpretation, $\absActs$ is interpreted over a finite set of elements called `actions' and the sorts $\cS_{ind}$ are interpreted over disjoint sets of concrete indexes. The array variables are assigned to  functions from these sets of indexes to the instance of the elements domain.


%
In what follows, we show how a specific abstract \ABMAS (simply called \ABMAS) can be used to model a \MASsystem as in Section~\ref{sec:mas}. To this end, we consider the different sorts and relations as in that section, and we encode the set of agent and environment templates. Instead of the abstract set $\uarr_{\absVar}$ of arrays, for each template $T_t$ with $t\in\set{1,\ldots,n,e}$ we consider a set $\uarr_{\Vars_t}$ of arrays, one for each variable in $\Vars_t$, that is used to store the current value of that variable for each concrete agent of type $t$. Intuitively, the `cell' for index $j$ 
in the array for variable $\var\in\Vars_t$ of template $T_t$ holds the value of $\var$ for the concrete agent with \id (in correspondence to) $j$ in the current global state (see Fig.~\ref{fig:encoding}). Since only one concrete environment exists, instead of arrays $\uarr_{\Vars_e}$ we use individual global variables $\env_{\Vars_e}$. Accordingly, the set  $\absActs$ of  generic actions is now $\cup_t \ActsA_t \cup \ActsAE$. 

Additional global variables $\ux$, which we now denote by $\globvars$ for readability, are needed for book-keeping (that is, to model any required low-level detail in the \MASsystem that are needed to encode its execution, such as flags, counters, turn indicators, etc). The global variable $\phase$, discussed in the next section, is an example of such variables. 

\begin{definition}
Given a \MASsystem $\cM:=\tup{\set{T_1,\ldots,T_n}, T_e,\Rels}$ and a set of initial states, an \defterm{\ABMAS} is a tuple: 
$$\tup{\Sigma,  \{\cS_{\IDS_t}\}_{t\in[1,n]}, \globvars, \{\uarr_{\Vars_t}\}_{t\in[1,n]}, \{ \arr_{\Acts_t}\}_{t\in[1,n]}, \iota,  \tau}$$
\noindent
\begin{inparaenum}[\it (i)]
\item $\Sigma:=\tup{\cS,\Rels, C}$, where $\cS$ are sorts (including the `actions' sorts $\cS_{\ActsA_t}$ for every $t$ and $\cS_{\ActsAE}$), $\Rels$ are the relation symbols of $\cM$  and $C$ a set of constants ($C$ includes a constant for each value $\valueof_t(l,\var)$ for every $t, l$ and $\var$);
\item   $\{\cS_{\IDS_t}\}_{t\in[1,n]}$ is a set of sorts of indexes type, one in correspondence of each $\IDS_t$.
\item $ \globvars$ is a set of individual variables used to encode the local state of the concrete environment  plus any book-keeping info (see later);
   \item $\{\uarr_{\Vars_t}\}_{t\in[1,n]}$ is a set of sets of arrays, one array for each variable $\var\in\Vars_t$ of each agent template $T_t$, whose elements range over $\Delta_\type$ for $\type=\dom(\var)$;
   \item  $\{ \arr_{\Acts_t}\}_{t\in[1,n]}$ is a set of arrays, one for agent template $T_t$, whose codomain has type 
   $\cS_{\ActsA_t}$ or $\cS_{\ActsAE}$;
  \item $\iota$ is an initial formula, whose individual variables are $ \globvars$ and whose array variables are  $ \uarr_{\Vars_t}, \arr_{\Acts_t}$;
  \item $\tau$ is a disjunction of transition formulae, with  individual variables $ \globvars, \globvars'$, and whose array variables are  $\uarr_{\Vars_t},  \arr_{\Acts_t},\uarr_{\Vars_t}',  \arr_{\Acts_t}'$.
  \end{inparaenum}
  \label{def:mas-as-abs}
\end{definition}

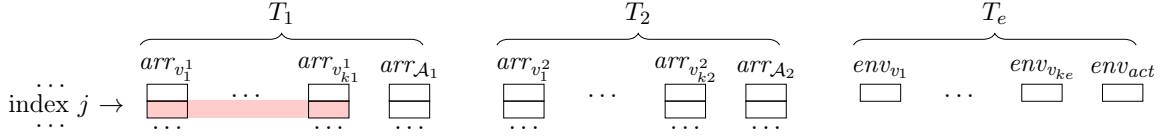
\begin{figure}[t]
\begin{center}
\resizebox{1\columnwidth}{!}{
 \begin{tikzpicture}

   \foreach \j in {1,2} {

  \begin{scope}[xshift=\j*5.3cm]
  
  \draw [decorate,decoration={brace,amplitude=5pt,raise=4ex}]
  (.8,-.2) -- (5,-.2) node[above,midway,yshift=2.2em]{$T_\j$};

  \foreach \x in {1,...,4} {
   \ifnum\x=2
   \node at (\x*1.2,-.3) {$\cdots$};
   \else 
   \ifnum\x=3
   \node at (\x*1.2,.1) {$\arr_{\var_{k\j}^{\j}}$};
   \else
   \ifnum\x=4
   \node at (\x*1.2,.1) {$\arr_{\Acts_\j}$};
   \else \node at (\x*1.2,.1) {$\arr_{\var_1^\j}$};
    \fi
    \fi
   \fi
   
   \foreach \y\lbl in {1,...,3} {
   \ifnum\x=2

   \else \ifnum\y=3
	\node (n\x\y) [minimum width=.6cm, minimum height=.25cm] at (\x*1.2,-\y*.26) {$\small{\cdots}$};
   \else 
       \node (n\x\y) [draw, minimum width=.6cm, minimum height=.25cm] at (\x*1.2,-\y*.25) {};
     \fi\fi
   }
   
  }

      \ifnum\j=1
  	\node (j) at (-.3,-.45) {index $j$ $\rightarrow$};
	\node (j1) at (-.5,-.15) {$\cdots$};
	\node (j2) at (-.5,-.75) {$\cdots$};
	
	\begin{pgfonlayer}{background}
        \node[ fill=red!20, inner sep=0pt, minimum height=.25cm, fit=(n12) (n32)] (circle-bg) {};
    \end{pgfonlayer}

	 \fi
  \end{scope}
    }

  \begin{scope}[xshift=3*5.3cm]
  
  \draw [decorate,decoration={brace,amplitude=5pt,raise=4ex}]
  (.8,-.2) -- (5,-.2) node[above,midway,yshift=2.2em]{$T_e$};

  \foreach \x in {1,...,4} {
   \ifnum\x=2
   \node at (\x*1.2,-.3) {$\cdots$};
   \else 
   \ifnum\x=3
   \node at (\x*1.2,.1) {$\env_{\var_{ke}}$};
   \else
   \ifnum\x=4
   \node at (\x*1.2,.1) {$\env_{\mathit{act}}$};
   \else \node at (\x*1.2,.1) {$\env_{\var_1}$};
    \fi
    \fi
   \fi
   
   \foreach \y\lbl in {1,...,1} {
   \ifnum\x=2

    \else 
       \node (n\x\y) [draw, minimum width=.6cm, minimum height=.25cm] at (\x*1.2,-\y*.25) {};
    \fi
   }
  }
  \end{scope}


 \end{tikzpicture}
 }
 \caption{A depiction of the encoding of agent templates.}
 \label{fig:encoding}
 \end{center}
 \end{figure}

Analogously to what done for abstract \ABMAS, we call a \emph{model} of an \ABMAS a FO interpretation of $\Sigma$ accounting for the `elements' domain, equipped with an assignment of the individual variables to elements of that interpretation; the action sorts $\cS_{\ActsA_t}$ and  $\cS_{\ActsAE}$ are resp. interpreted over the disjoint sets  $\ActsA_t$ and $\ActsAE$; the index sorts $\{\cS_{\IDS_t}\}_{t\in[1,n]}$ are interpreted over the disjoint sets of concrete agents \id{s} $\IDS_t$, for every $t\in[1,n]$. The array variables are assigned to  functions from these sets of \id{s}  to the elements domain of the model. In a model of an \ABMAS the standard notion of validity of formulae (like $\iota$ or $\tau$) is defined.

We now turn to show in details the encoding of a \MASsystem into an \ABMAS as just defined. 
%
Given the above, the initial formula is trivial to write (it has the very same shape given before). 
For transition formulae, the encoding differs (slightly) depending on the execution semantics assumed (see Section~\ref{sec:sync} and Section~\ref{sec:async}). Hence we talk about \sync and \async \ABMAS. 

The way in which the templates themselves are encoded is similar for both cases: as shown in Figure~\ref{fig:encoding}, for each $t\in[1,n]$ there are $k_t$ arrays for local variables $\set{\var_1^t,\cdots,\var_{k_t}^t}=\Vars_t$ plus one array storing the current chosen action (in $\ActsAE\cup\ActsA_t$) or $\nop$ as default value. The concrete environment is instead modelled with (a subset of the) global variables: there is one global variable for each template variable and one action variable storing the current synchronisation action in $\ActsAE$ (or $\nop$ as default value) 
 The transition formula $\tau$ in Def.~\ref{def:mas-as-abs} is obtained as the disjunction of simpler transition formulae $\hat\tau$. We start with the \async case.


\subsection{Encoding \async \MASsystem}
\label{sec:encoding_async}

Each global transition of the \MASsystem is encoded as a sequence of `steps' of the \ABMAS, each specified by a disjunction of transition formulae,  ordered by means of an additional global variable $\phase$ used to guide the progression. This is intuitively shown in Figure~\ref{fig:phases}, where nodes correspond to phases and edges to (disjunctions of) transition formulae. 
There are two kinds of progressions, corresponding to the two semantics in Def.~\ref{def:legal_global_transition_2}: either some non-empty subset of concrete agents execute local actions on their local state (upper branch) or a synchronization action is performed by the environment and at least one concrete agent (lower branch). The former case is realised by a non-empty sequence of steps following formula (\ref{eq:Local0_async}), in which local actions are written (`declared') in the appropriate position $j$ of array $\arr_{\Acts_t}$ for some $t$, followed by a single step in which the local state of all concrete agents that declared an action is updated in bulk (by applying the function in the $\lambda$-abstraction) as in (\ref{eq:Local1_async}). As transitions ($\hat\tau$ formulae) are taken nondeterministically, this will capture all possible sequences of (\ref{eq:Local0_async})-steps followed by a (\ref{eq:Local1_async})-step. The latter case is analogous: the formula in (\ref{eq:Global0_async}) makes sure that the environment and at least one concrete agent with \id $j$ and type $t$ can execute a synchronization action, which is then written in a global variable $\env_{\mathit{act}}$ as well as in the array position $\arr_{\Acts_t}[j]$. Then, a number of concrete agents can declare the same action, updating their action array as specified by (\ref{eq:Global1_async}). Nondeterministically, a bulk update is performed as in formula (\ref{eq:Global2_async}), which also updates the environment. In both cases, when the initial phase is reached again, the action arrays $\arr_{\Acts_t}$ of each template $T_t$ are reset to contain $\nop$ values. As a result a possible evolution of an \ABMAS template corresponds to a possible path in this intuitive diagram. 

\begin{figure}
\centering
\resizebox{.4\columnwidth}{!}{
\begin{tikzpicture}[state/.style=draw,circle,->,>=stealth,shorten >=1pt,auto,semithick,minimum size=0cm,inner sep=1pt]

\scriptsize
 
\begin{scope}[shift={(0cm,0cm)}]
\node[state, initial left] (A) {$0$};
\node[state] (B) at (2,.8) {$L$};
\node[state] (C) at (2,-.8) {$S$};

\draw[rounded corners=3pt] (A) |- node[left] {$(\ref{eq:Local0_async})$}  (B);
\draw[rounded corners=3pt] (A) |- node[left] {$(\ref{eq:Global0_async})$}  (C);

\path[->] 
(B) edge[] node [right,yshift=-.1cm] {$(\ref{eq:Local1_async})$}  (A)
(B) edge[loop right] node [right] {$(\ref{eq:Local0_async})$}  (B)
(C) edge[loop right] node [right] {$(\ref{eq:Global1_async})$}  (C)
(C) edge[] node [right,yshift=.1cm] {$(\ref{eq:Global2_async})$}  (A)
;
\end{scope}

\end{tikzpicture}
}
\caption{Intuitive progression of an \async \ABMAS.}
\label{fig:phases}
\end{figure}

We now list the transition formulae, numbered as in Figure~\ref{fig:phases}.
For encoding the first step of the upper branch we use a disjunction of transition formulae as (\ref{eq:Local0_async}) below, for each $t\in\set{1,\ldots,n,e}$ and local action $\act\in\ActsA_t$  
%

%
%
%
%
\begin{equation}\label{eq:Local0_async}
\small 
\begin{pmatrix}
\phase=0 \lor{}\\ \phase=L
\end{pmatrix}
\land \exists j_{\self} 
\begin{pmatrix}
 \arr_{\Acts_t}[j_{\self}]=\nop \;  \land \; \overline{P_t}(\act) \; \land{} \; \\
  \phase':=L  \; \land \; \arr_{\Acts_t}'[j_{\self}]:=\act
\end{pmatrix}
\end{equation}

\noindent
where a disjunction is used for compactness: we can write this as two distinct formulae.

\smallskip
\textit{Remark 2}. Above we denoted by $\overline{P_t}(\act)$ the transformation of the (quantifier-free) \Eformula $P_t(\act)=\anEformula$ as in Section~\ref{sec:mas} into a corresponding (quantifier-free) formula $\phi(\uj, \ux,\ua)$ for  \ABMAS{s}, with the same syntactic shape used in state-formulas. This is formally required because \Eformula{e}, e.g., make use of local template variables, whereas array-based  formulas make use of individual and array variables. Such transformation is trivial and it is not formalized further: it simply amounts to replace the variables of sort \id in $\varphi$ with elements of the sort of index types (essentially, to replace the index variables in the formula with index values from the sort used for agents in the \ABMAS formalization). 
As a special case, we impose that \self is always replaced with a special index $j_{\self}$, which is existentially quantified in the transition formula above:  there must exist an agent with index $j_{\self}$ and type $t$ that we can take as \self to evaluate $P_t(\act)$. If other index variables $j$ are present, these are existentially quantified as well.

\smallskip
The step above (see the loop on state $L$ in Fig.~\ref{fig:phases}) is repeated an unbounded number of times, as long as a new index $j_{\self}$ exists. Then, nondeterministically, a further step can be executed, characterised by the transition formula below. Here, as in the remainder of the section, we write $\uarr_{\Vars_t}[j]=l$ as a shorthand 
to denote that, for each variable $\var\in\Vars_t$ for some $t\in\set{1,\ldots,n,e}$, $\arr_\var[j]=\valueof(l, \var)$, namely the set of arrays $\uarr_{\Vars_t}[j]$ for the concrete agent with \id $j$ of type $t$ encode the local state $l\in L_t$ (see Fig.~\ref{fig:encoding}, in red). The same for $\uarr_{\Vars_t}':=\lambda j.~ \mathtt{case~of}~ \{\dots, \kappa_i:val(l,v),\dots\}$. 
The transition formula performs a bulk update of all instances (indexes $j$) 
by applying the transition function of each $T_t$ 
(below, one case is listed for each couple of local state-action): 

\begin{equation}\label{eq:Local1_async}
\small
 \begin{array}{@{}l@{}}
\phase=L \land \phase':=0 \land \bigwedge\limits_{t\in\set{1,\ldots,n,e}}  \arr_{\Acts_t}' := \lambda j.~ \nop \land{} \\
\bigwedge\limits_{t\in\set{1,\ldots,n,e}} \uarr_{\Vars_t}':=\lambda j.\begin{pmatrix}\mathtt{case~of}~ \\
\begin{Bmatrix}
\uarr_{\Vars_t}[j]=l_1 \land \arr_{\Acts_t}[j]=\act_1:\delta_t(l_1, \act_1)\\
\cdots \\ 
\uarr_{\Vars_t}[j]=l_m \land \arr_{\Acts_t}[j]=\act_k:\delta_t(l_m, \act_k)
\end{Bmatrix}
\end{pmatrix}
\end{array}
\end{equation}
Above, for each $j$ one possible case applies. For instance, if $\act_1$ was declared and the state is $l_1$, then $\uarr_{\Vars_t}[j]=\delta_t(l_1, \act_1)$. 
%
%

For synchronization actions, for each action $\act\in\ActsAE$ and template $T_t$ we have the following formula, making sure that at least one concrete agent and the environment can perform the same action, then written in the global variable $\env_{\mathit{act}}$:

\begin{equation}\label{eq:Global0_async}
\small
\phase=0 \land \exists j_{\self} \begin{pmatrix}
 \arr_{\Acts_t}[j_{\self}]=\nop \land  \env_{\mathit{act}}=\nop \land{} \\
 \overline{P_t}(\act) \land \overline{P_e}(\act) \land{}\\
 \arr_{\Acts_t}'[j_{\self}]:=\act \land \env_{\mathit{act}}':=\act \land \phase':=S  
\end{pmatrix}
\end{equation}

Then two possibilities exist: either other concrete agents participate in the synchronization an unbounded number of times (via (\ref{eq:Global1_async})), or the bulk progression is performed by applying the formula (\ref{eq:Global2_async}) further below. Again, this is encoded as a disjunction of formulae. For each $t\in[1,n]$ and $\act \in\ActsAE$:  

%
\begin{equation}\label{eq:Global1_async}
\small
 \begin{array}{@{}l@{}}
\phase=S  \land \exists j_{\self} 
\begin{pmatrix}
\env_{\mathit{act}}=\act \land \overline{P_t}(\act) \land \arr_{\Acts_t}[j_{\self}]=\nop \\ 
  {}\land \phase':= S \land  \arr_{\Acts_t}'[j_{\self}]:=\act
\end{pmatrix}
\end{array}
\end{equation}
%
%
\begin{equation}\label{eq:Global2_async}
\small
 \begin{array}{@{}l@{}}
\phase=S \land  \env_{\mathit{act}}=\act \land \env_{\Vars_e}=l_e\land \env_{\Vars_e}':=\delta_e(l_e, \act) \\ 
{}\land  \env_{\mathit{act}}':=\nop \land \arr_{\Acts_t}' := \lambda j.~ \nop \land \phase':=0 \land{}  \\
{}\bigwedge\limits_{t\in[1,n]} \uarr_{\Vars_t}':=\lambda j.\begin{pmatrix}\mathtt{case~of}~ \\
\begin{Bmatrix}  \uarr_{\Vars_t}[j]=l_1  \land \arr_{\Acts_t}[j]=\act : \delta_t(l_1, \act)\\
\cdots \\ 
\uarr_{\Vars_t}[j]=l_m  \land \arr_{\Acts_t}[j]=\act : \delta_t(l_m, \act)\}
\end{Bmatrix}
\end{pmatrix}

\end{array}
\end{equation}

\medskip
\begin{example}\textit{(cont.d}) The running scenario can be easily encoded, assuming a \sync semantics, through the transition formulae above (see Section~\ref{sec:implem} for the actual encoding for the model checker). 
The only detail we did not discuss so far, which is needed for the scenario, is the rule that imposes that the cannon and the attacking robots alternate their moves. This is trivially encoded by means of an additional ``book-keeping", boolean, \emph{global} variable $turn$, so that the protocol of each action of both templates is augmented with an additional condition that tests such variable. Then, the value of the variable is toggled  in the bulk update, i.e., in the transitions conforming to equation (\ref{eq:Global2_async}). 
\end{example}


\subsection{Encoding \sync \MASsystem}
\label{sec:encoding_sync}

As done for the \async case, each global transition of the \MASsystem is encoded as a sequence of `steps' of the \ABMAS, each specified by a disjunction of transition formulae,  ordered by means of an additional global variable $\phase$ used to guide the progression. This is intuitively shown in Figure~\ref{fig:phases_conc}, where nodes correspond to phases and edges to (disjunctions of) transition formulae. 

As shown, there are again two kinds of progressions, corresponding to the two cases in Def.~\ref{def:legal_global_transition}. The upper cycle represent the execution of some local action (by all agents and environment that are able to do so), while all others execute $\nop$. The lower cycle encodes the case in which a synchronisation action is performed by the environment and all the concrete agents that can execute that action (but at least one), while all others execute $\nop$. 

As depicted, the former case is realised by a sequence of steps: 
\begin{compactitem}
\item when the current phase is $0$, a sequence of steps is applied. Each of these is encoded by formula (\ref{eq:Local0}), in which a local action is written in the appropriate position $j$ of array $\arr_{\Acts_t}$ for some $t$. This encodes the fact that the concrete agent of type $t$ and \id $j$ selected one action to execute. This step can be executed an arbitrary number of times, but only once for each concrete agent and for the environment;
\item when the current phase is $L$ and for \emph{every} concrete agent and the environment it is true that either the one local action was selected or that all local actions are not executable, then one further step is enabled and the overall process reaches phase $L_2$. This is captured by formula (\ref{eq:localSync1});
\item when the current phase is $L_2$, only one possible step can be taken: all concrete agents that selected an action are updated in bulk and the local action each of them selected is executed (by applying the function in the $\lambda$-abstraction) as in (\ref{eq:Local1}).
\end{compactitem}

Note that transitions ($\tau$ formulae) are taken nondeterministically, hence this will capture all possible sequences of (\ref{eq:Local0})-steps until the guard of (\ref{eq:localSync1}) is satisfied, followed by a (\ref{eq:localSync1})-step and a (\ref{eq:Local1})-step. 
The latter case (bottom cycle) is analogous: 

\begin{compactitem}
\item the formula in (\ref{eq:Global0}) makes sure that the environment and at least one concrete agent with \id $j$ and type $t$ can execute a synchronisation action, which is then written in a global variable $\env_{\mathit{act}}$ as well as in the array position $\arr_{\Acts_t}[j]$;
\item then, concrete agents can select the same action, updating their action array as specified by (\ref{eq:Global1});
\item in a further step, that is only enabled when for \emph{every} concrete agents it is true that either $\env_{\mathit{act}}$ is not executable or was indeed selected, the phase is progressed to $S_2$. This is captured by formula (\ref{eq:localSync2});
\item when the current phase is $S_2$, a bulk update is finally performed as in formula (\ref{eq:Global2}), which also updates the environment. 
\end{compactitem}

In both cycles above, when the initial phase $0$ is reached again, the action arrays $\arr_{\Acts_t}$ of each template $T_t$ are reset to contain $\nop$ values. As a result a possible evolution of an \ABMAS template corresponds to a possible path in this intuitive diagram, chosen nondeterministically.

\begin{figure}
\centering
\resizebox{.4\columnwidth}{!}{
\begin{tikzpicture}[state/.style=draw,circle,->,>=stealth,shorten >=1pt,auto,semithick,minimum size=0cm,inner sep=1pt]

\scriptsize
 
\begin{scope}[shift={(4cm,0cm)}]
\node[state, initial left] (A) {$0$};
\node[state] (B) at (2.4,1) {$L$};
\node[state] (C) at (2.4,-1) {$S$};

\node[state] (D) at (1.2,.5) {$L_2$};
\node[state] (E) at (1.2,-.5) {$S_2$};

\draw[rounded corners=3pt] (A) |- node[left] {$(\ref{eq:Local0})$}  (B);
\draw[rounded corners=3pt] (A) |- node[left] {$(\ref{eq:Global0})$}  (C);

\path[->] 
(D) edge[] node [above,pos=.7,yshift=-.05cm] {$(\ref{eq:Local1})$}  (A)
(B) edge[loop right] node [right] {$(\ref{eq:Local0})$}  (B)
(C) edge[loop right] node [right] {$(\ref{eq:Global1})$}  (C)
(E) edge[] node [below,pos=.7,yshift=.05cm] {$(\ref{eq:Global2})$}  (A)

(B) edge[] node [below, yshift=.13cm, pos=.3] {$(\ref{eq:localSync1})$}  (D)
(C) edge[] node [above, yshift=-.13cm, pos=.3] {$(\ref{eq:localSync2})$}  (E)
;
\end{scope}

\end{tikzpicture}
}
\caption{Intuitive progression of a \sync \ABMAS.}
\label{fig:phases_conc}
\end{figure}

We now list the transition formulae. 
For encoding the first step of the upper branch we use a disjunction of transition formulae as (\ref{eq:Local0}) below, for each $t\in\set{1,\ldots,n,e}$ and local action $\act\in\ActsA_t$  
(a disjunction is also used in (\ref{eq:Local0}) for compactness: we can write this as two distinct formulae).  
By these formulae a concrete agent declares action $\act$ (recall that primed variables denote updates): 
%
%
%
%
\begin{equation}\label{eq:Local0}
\small
\begin{pmatrix}
\phase=0 \lor{}\\ \phase=L
\end{pmatrix}
\land \exists j_{\self} 
\begin{pmatrix}
 \arr_{\Acts_t}[j_{\self}]=\nop \;  \land \; \overline{P_t}(\act) \; \land{} \; \\
  \phase':=L  \; \land \; \arr_{\Acts_t}'[j_{\self}]:=\act
\end{pmatrix}
\end{equation}
\noindent
The same remark for the \async case applies. 

The step above (see the loop on state $L$ in Figure~\ref{fig:phases}) can be repeated an unbounded number of times, as long as a new index $j_{\self}$ exists. The guard of the following formula requires that for every $j$ either the selected action is not equal to $\nop$, i.e. no action was selected, or that no local action is executable:
\begin{equation}\label{eq:localSync1}
\small
\phase=L \land \forall j \begin{pmatrix}
 {}\bigwedge\limits_{t\in \set{1,\ldots,n,e}} \arr_{\Acts_t}[j]\neq\nop \lor{} \\ {}\bigwedge\limits_{\act\in\ActsA_t}  \neg\overline{P_t}(\act) \land \phase':=L_2  
\end{pmatrix}
\end{equation}

Then a further step becomes enabled, characterised by the transition formula below. Here, as in the remainder of the section, we write $\uarr_{\Vars_t}[j]=l$ as a shorthand 
to denote that, for each variable $\var\in\Vars_t$ for some $t\in\set{1,\ldots,n,e}$, $\arr_\var[j]=\valueof(l, \var)$, namely the set of arrays $\uarr_{\Vars_t}[j]$ for the concrete agent with \id $j$ of type $t$ encode the local state $l\in L_t$ (see Figure~\ref{fig:encoding}, in red). 
 The same for $\uarr_{\Vars_t}':=\lambda j.~ \mathtt{case~of}~ \{\dots, \kappa_i:val(l,v),\dots\}$. 
The transition formula performs a bulk update of all concrete agents (indexes $j$) that already declared an action (or have $\nop$ in their cell),  depending of the type $t$ (below, one case is listed for each couple of local state and action, for each template $T_t$):
%
\begin{equation}\label{eq:Local1}
\small
 \begin{array}{@{}l@{}}
\phase=L_2 \land \phase':=0 \land \bigwedge\limits_{t\in\set{1,\ldots,n,e}}  \arr_{\Acts_t}' := \lambda j.~ \nop \land{} \\
\bigwedge\limits_{t\in\set{1,\ldots,n,e}} \uarr_{\Vars_t}':=\lambda j.\begin{pmatrix}\mathtt{case~of}~ \\
\begin{Bmatrix}
\uarr_{\Vars_t}[j]=l_1 \land \arr_{\Acts_t}[j]=\act_1:\delta_t(l_1, \act_1)\\
\cdots \\ 
\uarr_{\Vars_t}[j]=l_m \land \arr_{\Acts_t}[j]=\act_k:\delta_t(l_m, \act_k)
\end{Bmatrix}
\end{pmatrix}
\end{array}
\end{equation}
Above, for each $j$ one possible case applies. E.g., if $\act_1$ was declared and the state is $l_1$, then $\uarr_{\Vars_t}[j]=\delta_t(l_1, \act_1)$. 
%
%

For synchronization actions, for each action $\act\in\ActsAE$ and template $T_t$ we have the following formula, which makes sure that at least one concrete agent and the environment can perform the same action, which is then written in the global variable $\env_{\mathit{act}}$:
\begin{equation}\label{eq:Global0}
\phase=0 \land \exists j_{\self} \begin{pmatrix}
 \arr_{\Acts_t}[j_{\self}]=\nop \land  \env_{\mathit{act}}=\nop \land{} \\
 \overline{P_t}(\act) \land \overline{P_e}(\act) \land \arr_{\Acts_t}'[j_{\self}]:=\act \land{}\\
  \env_{\mathit{act}}':=\act \land \phase':=S  
\end{pmatrix}
\end{equation}
After the phase is updated to $S$, either other concrete agents participate in the synchronisation (transition formula (\ref{eq:Global1})) or, if no more agent can select action $\act$, the phase is first progressed to $S_2$ (formula (\ref{eq:localSync2})) and then a bulk update of local states is performed by applying the formula (\ref{eq:Global2}) further below. For each $t\in[1,n]$ and $\act \in\ActsAE$:  
\begin{equation}\label{eq:Global1}
\small
 \begin{array}{@{}l@{}}
\phase=S  \land \exists j_{\self} 
\begin{pmatrix}
\env_{\mathit{act}}=\act \land \overline{P_t}(\act) \land \arr_{\Acts_t}[j_{\self}]=\nop \\ 
  {}\land \phase':= S \land  \arr_{\Acts_t}'[j_{\self}]:=\act
\end{pmatrix}
\end{array}
\end{equation}
Then: 
\begin{equation}\label{eq:localSync2}
\small
\phase=S \land \forall j \begin{pmatrix}
 {}\bigwedge\limits_{t\in \set{1,\ldots,n,e}} \arr_{\Acts_t}[j]\neq\nop \lor{} \\  \neg\overline{P_t}(\act) \land \phase':=S_2  
\end{pmatrix}
\end{equation}
%
%
%

\noindent
Finally, a bulk update is executed, which resets the phase and applies the effect of $\act$ for each agent (that declared such action instead of $\nop$) and for the environment: 

\begin{equation}\label{eq:Global2}
\small
 \begin{array}{@{}l@{}}
\phase=S_2 \land  \env_{\mathit{act}}=\act \land \env_{\Vars_e}=l_e\land \env_{\Vars_e}':=\delta_e(l_e, \act) \\ 
{}\land  \env_{\mathit{act}}':=\nop \land \arr_{\Acts_t}' := \lambda j.~ \nop \land \phase':=0 \land{}  \\
{}\bigwedge\limits_{t\in[1,n]} \uarr_{\Vars_t}':=\lambda j.\begin{pmatrix}\mathtt{case~of}~ \\
\begin{Bmatrix}  \uarr_{\Vars_t}[j]=l_1  \land \arr_{\Acts_t}[j]=\act : \delta_t(l_1, \act)\\
\cdots \\ 
\uarr_{\Vars_t}[j]=l_m  \land \arr_{\Acts_t}[j]=\act : \delta_t(l_m, \act)\}
\end{Bmatrix}
\end{pmatrix}
\end{array}
\end{equation}

\section{Verification}
\label{sec:results}

We denote by $\ABM$ the \ABMAS obtained by encoding a \MASsystem $\M$ as in the previous section (so we will refer to \sync and \async \ABMAS). 
A \defterm{safety formula} for $\ABM$ is a state formula
$\phi$, as defined in Section~\ref{sec:mas-abs}, of the form $\exists \uj. \phi(\uj, \ux,\ua)$. These formulae are used to characterise undesired states of $\ABM$. 

By adopting a customary terminology for array-based systems, we say that $\ABM$ is \defterm{safe with
        respect to} $\phi$ if intuitively the system has no finite run leading from
$\iota$ to $\phi$.  Formally, this means that there is no interpretation $\I_0$ of relations and no possible assignment to the individual and array variables $\ux^0,\ua^0, \ldots, \ux^k, \ua^k$ such
that the formula
\begin{equation*}\label{eq:smc1}
  \begin{array}{@{}l@{}}
    \iota(\ux^0, \ua^0)
    \land \tau(\ux^0,\ua^0, \ux^1, \ua^1)
    \land \cdots \land 
     \tau(\ux^{k-1},\ua^{k-1}, \ux^k,\ua^{k})
    \land \phi(,\ux^k,\ua^{k})
  \end{array}
\end{equation*}
is valid in any model of $\ABM$. 
The safety problem for $\ABM$ is the following: 
\emph{Given a safety formula $\phi$ as before, decide whether $\ABM$ is safe with respect to $\phi$}. 
It is immediate to see that this matches Def.~\ref{def:reachable}: the \ABMAS cannot be safe w.r.t  $\phi$ if there exists an initial interpretation $\I_0$ so that a global state $g$ with $g\models_{\I_0} \psi$, where $\psi$ is an \Eformula with $\overline{\psi}=\phi$, is reachable  through a run $\tup{\rho,\I_0}$, and vice-versa (recall the description of $\bar{\cdot}$ in Remark~2).

\begin{theorem}
An \Eformula  $\goal$ is reachable in a \MASsystem $\M$ iff the corresponding \ABMAS $\ABM$ is unsafe w.r.t $\overline{\goal}$. 
\label{thm:reachable_iff_unsafe}
\end{theorem}


\begin{example}\emph{(cont.d)} Let $\goal=\exists j. (\mytexttt{loc}^{[j]}=\mytexttt{target})$: \emph{at least} one agent can reach the target location. In both \sync and \async cases, the \MASsystem is unsafe if the protocol of the cannon is arbitrary, as it may adopt a naive behavior and not defend the target location. 

	Even if we assume that the cannon behaves according to the plan we described when introducing the running scenario (see Sec.~\ref{sec:introduction}), then a \sync \MASsystem is safe. As anticipated, however, an \async \MASsystem is unsafe because it is enough for two robots (or more) to move to a waypoint, then to the target, for winning the scenario. Indeed, when two or more robots reach a waypoint and the cannon fires at them, there are no guarantees that all these robots will be destroyed. For this reason, the same conclusion applies for further goal formulae such as $\goal=\exists j_1,j_2. (\mytexttt{loc}^{[j_1]}=\mytexttt{target} \land \mytexttt{loc}^{[j_2]}=\mytexttt{target} \land j_1\neq j_2)$: a run with three robots exists. 
\end{example}


\subsection{Soundness and Completeness}
\label{sec:theorems}

The algorithm described in Figure~\ref{fig:algorithm} shows the \emph{SMT-based backward reachability procedure} (or,
\defterm{backward search}) for handling the safety problem for an \ABMAS $\ABM$.  An integral
part of the algorithm is to compute \textit{symbolic} preimages.
For that purpose,  for any $\tau(\uz,\uz')$ and $\phi(\uz)$ (where $\uz'$ are renamed copies of $\uz$),
we define $\textit{Pre}(\tau,\phi)$ as the formula
$\exists \uz'(\tau(\uz, \uz')\land \phi(\uz'))$.
The \emph{preimage} of the set of states described by a state formula
$\phi(\ux)$ is the set of states described by
$\textit{Pre}(\tau,\phi)$ (notice that, when $\tau=\bigvee\hat\tau$,
        then $\textit{Pre}(\tau,\phi)=\bigvee\textit{Pre}(\hat\tau,\phi)$). 
Backward search computes iterated preimages of a safety formula $\phi$, 
until a fixpoint is reached (in that case, $\ABM$ is \emph{safe} w.r.t. $\phi$) or until a set
intersecting the initial states (i.e., satisfying $\iota$) is found (in that case, $\ABM$ is \emph{unsafe} w.r.t. $\phi$) . 
%
%
\textit{Inclusion} (Line~\ref{algo:while}) and \textit{disjointness} 
(Line~\ref{algo:return1}) tests can be discharged via proof obligations to be handled by SMT solvers. 
The fixpoint is reached when the test in Line~\ref{algo:while} returns \textit{unsat}: the preimage of the set of the current states is included in the set of states reached by the backward search so far. 
The test at Line~\ref{algo:return1} is true when the states visited so far by the backward search (represented as the iterated application of preimages to the safety formula $\phi$) includes a possible initial state (i.e., a state satisfying $\iota$). If this is the case, then $\ABM$ is unsafe w.r.t. $\phi$. 

The procedure either does not terminate or returns a \safe/\unsafe result.  
Given a \MASsystem $\M$ and a safety formula $\goal$, a \safe (resp. \unsafe) result is \emph{correct} iff $\ABM$ is safe (resp. unsafe) w.r.t. $\goalab$. 
%


\begin{definition}
Given a \ABMAS $\ABM$ and a safety formula $\goalab$, a verification procedure for checking (un)safety is: 
\begin{inparaenum}[\it (i)]
\item \defterm{sound} if, when it terminates, it returns a correct result;
\item \defterm{partially sound} if a \safe result is always correct;
\item \defterm{complete} if, whenever \unsafe is the correct result, then \unsafe is indeed returned.
\end{inparaenum}
\label{def:properties}
\end{definition}

The same definition can be given for \MASsystem{s}. 

If partially sound, an \unsafe result may be not correct due to ``spurious''  traces. 
Also, as the preimage computation can diverge on a safe system, 
there is no guarantee of termination.  
%
However, since \ABMAS are a special case of the abstract \ABMAS (Def.~\ref{def:abstract-ABMAS}), we can adapt known results on array-based systems (see \cite{lmcs,CGGMR19}). 


\begin{theorem}\label{thm:sound-complete}
Backward search for the safety problem is \emph{sound} and \emph{complete} 
for \emph{\async} \ABMAS.
\end{theorem}

\begin{proof}[Proof Sketch] 
For soundness we need suitable decision procedures for the satisfiability tests 
in Alg.~\ref{fig:algorithm}. By taking inspiration from \cite{lmcs,CGGMR19}, we can show that the preimage of a state formula can be converted to a state formula and that entailment can be decided via finite instantiation techniques. For completeness, finite unsafe traces are found after finitely many steps. The full proof is reported in Appendix~\ref{sec:app1}.
\end{proof}

\begin{theorem}\label{thm:sound-only}
Backward search for the safety problem is \emph{partially sound} and \emph{complete} 
for \emph{\sync} \ABMAS. 
\end{theorem}

\begin{proof}[Proof Sketch] The proof is analogous to the previous one, with a significant difference due to the presence of universal quantification in transition formulae (see Sec.~\ref{sec:encoding_sync}): for this reason, the proof should be adapted following the line of reasoning from~\cite{GhilUniv}.
\end{proof}

It follows that, given Theorem~\ref{thm:reachable_iff_unsafe}, \myi for \sync \MASsystem we have partial soundness due to the need of using universal quantification in transition formulae (see Sec.~\ref{sec:encoding_sync}), and \myii ours is a sound and complete procedure for checking (un)safety of \async \MASsystem as only existential quantification is needed (see Sec.~\ref{sec:encoding_async}). 
Backward search for \async \ABMAS is thus a \emph{semi-decision} procedure. 
%
Universal quantification proves to be crucial, as available results on verification of array-based systems \cite{lmcs,CGGMR19} cannot guarantee that in fact \emph{all} indexes are considered, leading to executions in which some indexes will be simply ``disregarded'' from that point on. 
This induces \emph{spurious} runs in the encoded \MASsystem, akin to a \emph{lossy} system \cite{lossy}. 
%

\begin{algorithm2e}[t]
\footnotesize
\SetAlgoNoEnd
\caption{backward search $\textit{BReach}(\ABM,\phi)$}
\label{fig:algorithm}
	%
	$B\longleftarrow \bot$; \\
	\While{$\phi\land \neg B$ is \emph{satisfiable}\label{algo:while}}{
		\textbf{if} ($\iota\land \phi$ is satisfiable) \textbf{then return}  $\unsafe$\label{algo:return1};\\
		$B\longleftarrow \phi\vee B$;\\  
		$\phi\longleftarrow Pre(\tau, \phi);$ ~// $\tau$ is as in Def.~\ref{def:mas-as-abs}\\
	}
\KwRet  $\safe;$
\end{algorithm2e}

\subsection{A Condition for Termination}
\label{sec:termination}

 
We show under which (sufficient) condition we can guarantee termination of the backward search, which will gives us a decision procedure for unsafety. 
%
%
Although technical proofs are quite
involved at the syntactic level, they can be intuitively understood
as based on this \emph{locality condition}: the states ``visited'' by the backward search can be represented by state-formulae which do not include direct/indirect comparisons and ``joins" of distinct state variables for different agent \id{s}. 
E.g., we cannot have $x^{[j_1]}=y^{[j_2]}$, which equates the current variable $x$ and $y$ of the concrete agents with \id{s} $j_1$ and $j_2$. We  can however write  $x^{[j_1]}=c$, $y^{[j_2]}=c$ for a constant $c$. 

Of course, if this property is true for $\phi$, it does not necessary hold for the formula obtained by ``regressing" $\phi$ w.r.t. some transition formula $\hat{\tau}$, i.e., $\textit{Pre}(\hat\tau,\phi)$ as those in Section~\ref{sec:mas-abs}:  $\hat\tau$ includes translations $\overline{P_t}(\act)$ of template protocols $P_t(\act)$ for action $\act$.   
%
%
Formally, we call a state formula \defterm{local} if it is a disjunction of the formulae of the form:
\begin{equation}\label{eq:local}
  \textstyle
  \exists j_1\cdots \exists j_m\, ( Eq(j_1,\dots, j_m)  \land
    \bigwedge_{k=1}^m \phi_k(j_k,\ux,\ua))
\end{equation}
Here, $Eq$ is a conjunction of
variable (dis)equalities, $\phi_k$ are quantifier-free formulae, and
$j_1,\ldots,j_m$ are individual variables of index sort. 
Note how 
each $\phi_k$ in~\eqref{eq:local} can contain only
the existentially quantified index variable $j_k$. 
As said before, this limitation necessarily has an impact on transition formulae as well. We say that a transition formula $\hat\tau$ is \emph{local} 
if whenever a formula $\phi$ is local, so is
$\textit{Pre}(\hat\tau,\phi)$. 

\begin{theorem}\label{thm:decid-local}
$\textit{BReach}(\ABM,\goalab)$ always \emph{terminates} if \myi the \ABMAS is \async and \myii its transition  formula $\tau=\bigvee\hat\tau$ is a disjunction of \emph{local} transition formulae and \myiii $\goalab$ is local.
\end{theorem}

\begin{proof}[Proof Sketch]
Although involved, the proof follows a similar argument to that of Thm.~6.2 in~\cite{CGGMR19}: it can be proved that an \ABMAS corresponds to a particular case of acyclic local RASs, therein defined, but without ``data'' and \emph{extended} with uninterpreted relations. The full proof is reported in Appendix~\ref{sec:app2}.  
\end{proof}

Notice that the structure of the proof of the previous theorem follows the schema of a proof from~\cite{MSCS20}, but there is a \emph{significant} difference: the decidability result  in \cite{MSCS20} is based on locality and applies to array-based systems whose FO-signatures do \emph{not} have relational symbols, whereas here our array-based systems \emph{do} have free relational symbols. This difference is reflected in the kind of FO-formulae that are involved.
 
The locality condition is `semantic' for transition formulae and syntactic for state formulae. However, since the updates in transition formulae in Sec.~\ref{sec:mas-abs} are assignment to constants,  if all $P_t(\act)$ are local then $\tau$  is local.   

It is immediate to verify that agent protocols and the goal formula in our running scenario are both local, hence checking safety is decidable.


\section{Implementation}\label{sec:implem}

In this section we illustrate and evaluate experimentally our tool illustrate our implementation approach, called \TOOLtitle, based on the third-party model checker MCMT \cite{GhilardiR10}. 
We organize this section as follows: first, in Section~\ref{sec:mcmt} we describe the MCMT model-checker and comment on the encoding of \MASsystem{s} into MCMT input files. Then, in Section~\ref{sec:tool}, we showcase our \TOOL tool for the automated encoding of \MASsystem{s} into an \ABMAS (in the format MCMT expects) and subsequent extraction of witnesses of unsafety, and in Section~\ref{sec:evaluation} we evaluate experimentally our overall approach and tool chain.  

\subsection{MCMT: Model Checker Modulo Theories}
\label{sec:mcmt}

MCMT \cite{GhilardiR10} is a declarative and deductive symbolic model checker for safety properties of infinite state systems, based on backward reachability and fix-points computations (with calls to an SMT solver). 
%
The input of the software is a simple, textual representation of an ABS, including \myi the declaration of individual and array variables, \myii the initial state formula, \myiii the goal state formula, \myiv the list of transition formulae (i.e. either formulae (\ref{eq:Local0_async})-(\ref{eq:Global2_async}) or formulae (\ref{eq:Local0})-(\ref{eq:Global2}) in Sec.~\ref{sec:mas-abs}). It is thus  sufficient to write a one-to-one textual counterpart of a given \ABMAS{s} as input to MCMT.

If the system is unsafe, a \emph{witness} is provided, from which one can extract unsafe runs and refine the model (see Section~\ref{sec:tool} on our implementation of this step). 
MCMT supports forms of universal quantification needed for \sync \MASsystem (see Sec.~\ref{sec:encoding_sync}), in which case a warning is issued, as this may generate spurious runs (see below Theorem~\ref{thm:sound-complete}). 

\subsubsection{MCMT input files for \async and \sync \ABMAS{s}}
\label{sec:mcmt_async}

We now exemplify the textual encoding by listing the salient portions of the input file for our running examples. With the exception of minor syntactic details, such textual encoding is the same for the \async and \sync cases, hence in this paper we focus on the former case only (the difference is simply represented by which transitions formulae we must encode to MCMT, i.e., either formulae (\ref{eq:Local0_async})-(\ref{eq:Global2_async}) or formulae (\ref{eq:Local0})-(\ref{eq:Global2}) in Sec.~\ref{sec:mas-abs}). 

The textual encoding starts with some definitions, which intuitively define the sorts that are used in the ABS. For the running example, these are:

{\scriptsize
\begin{verbatim}
:smt (define-type Action)
:smt (define-type PhaseSort)
:smt (define-type BOOLE)
:smt (define-type StringLoc)
:smt (define-type turnSort)
 \end{verbatim}}
 
Intuitively, we need (distinct) sorts for actions, the phase flag used to guide the progression (see Section~\ref{sec:mas-abs}), booleans, locations (for robot positions, blasts and the pulse) and turns (as our example is turn-based, alternating the moves of the cannon the moves of all robots - see the example at the end of Section~\ref{sec:encoding_async}). 

Then, constants and relations are declared with their corresponding sort. 

{\scriptsize
\begin{verbatim}
:smt (define Snow ::(-> StringLoc StringLoc ))
:smt (define pulseA ::Action)
:smt (define pulseB ::Action)
:smt (define gotoA ::Action)
:smt (define gotoB ::Action)
:smt (define goTargetA ::Action)
:smt (define goTargetB ::Action)
:smt (define blastA ::Action)
:smt (define blastB ::Action)
:smt (define TRUE ::BOOLE)
:smt (define FALSE ::BOOLE}
:smt (define A ::StringLoc)
:smt (define B ::StringLoc)
:smt (define init ::StringLoc)
:smt (define target ::StringLoc)
:smt (define P0 ::PhaseSort)
:smt (define PL ::PhaseSort)
:smt (define PS ::PhaseSort)
:smt (define turnTEMPS ::turnSort)
:smt (define turnREST ::turnSort)
 \end{verbatim}}

After some further declarations needed by MCMT, local and global variables are declared. In our encoding, consistently with the encodings illustrated in Section~\ref{sec:mas-abs}, local variables are used for the template variables and global variables are used for both the environment variables and additional book-keeping data. 

For instance, for encoding our running example we include: 

{\scriptsize
\begin{verbatim}
:local locATT StringLoc   
:local destroyedATT BOOLE
:local actATT Action
:global pulseLoc StringLoc
:global actEnv Action
:global phase PhaseSort
 \end{verbatim}}
 
The three arrays (the local variables) are used, respectively, to hold the location of robots, their destroyed (boolean) state, the actions they currently selected for execution. 
The remaining three (global) variables are used to hold the location towards which the pulse is currently directed, the action selected for execution by the environment, the phase value needed for encoding the transitions. 
 
Further, the initial states are represented by means of a simple formula, written as a conjunction of conditions: 

{\scriptsize
\begin{verbatim}
:initial
:var x
:cnj (= pulselocE NULL_StringLoc) (= actE NULL_Action) (= locATT[x] init) (= destroyedATT[x] FALSE) 
(= actATT[x] NULL_Action) (= turn turnTEMPS) (= phase P0) 
 \end{verbatim}} 
 
\noindent
The conditions above require that: the EMP pulse is not directed towards any waypoint, the environment has declared no action, all robots are in their initial state (here, \texttt{x} is implicitly quantified universally), no robot is destroyed, no robot declared yet an action, the turn is of the cannon, the phase is equal to $0$. Notice that nothing is said about the $Snow$ relation, as we want to check that the system is safe irrespective of the snow conditions. 

Further, we specify the safety formula to verify. For our example, this formula requires that at least one robot exists (we use an index \texttt{z1}, implicitly quantified existentially) so that the cell of array \texttt{locATT} has value equal to \texttt{target}, namely the robot is in the target location: 

{\scriptsize
\begin{verbatim}
:u_cnj (= locATT[z1] target) 
 \end{verbatim}} 
 
Transition formulae are encoded directly. For instance, the transition formula $(\ref{eq:Local0_async})$, for the local action $\mytexttt{gotoB}$ in the running scenario (in the \async case), is encoded in MCMT as two formulae: one for $\phase=0$ and one for $\phase=L$. 
The guard is composed of six conjuncts, where the index variable $\texttt{x}$ is used to express conditions on the same concrete agent: the first two correspond to $\phase=0$ and $\arr_{\Acts_t}[j_{\self}]=\nop$ in $(\ref{eq:Local0_async})$, and the rest encode the template's protocol function, i.e. correspond to $\overline{P_t}(\act)$ in $(\ref{eq:Local0_async})$. 
Intuitively, they require that the robot is in the initial location, is not destroyed, there is no pulse directed at the waypoint, the path is clear of snow. 
The keyword $\mytexttt{:numcases}$ specifies how many outcomes this $\mytexttt{:transition}$ has. In this case two: one for the index $j$ that is taken as \texttt{x} (which is arbitrarily selected) and one for the rest. This reconstruct the existential quantification $\exists j_{\self}$ in $(\ref{eq:Local0_async})$. In the first case, which is applied to the agent with index $j_{\self}$, the array variable $\mytexttt{actROBOTS[j]}$ is updated to $\mytexttt{gotoB}$ and $\mytexttt{phase}$ to $\mytexttt{L}$ while the others remain unchanged (assignments are encoded by simply listing, in the same order as variable definitions above, either a new value, thus expressing an update, or by simply repeating the variable name if its value is not updated). In the latter case only the phase is updated, as this case applies for all agents that are assigned an index other than $j_{\self}$. 

{\scriptsize
\begin{verbatim}
:transition  
:var j   
:var x
:guard 
   (= phase 0)
   (= actATT[x] Nop_Action)
   (= locATT[x] init)
   (= destroyedATT[x] FALSE)
   (not (= pulseLoc B))
   (= Snow (start B) FALSE)
   
   
   
   
   
   
:numcases 2
:case (= x j)
   :val locATT[j]  
   :val destroyedATT[j]   
   :val gotoB
   :val pulseLoc   
   :val actEnv   
   :val L
:case 
   :val  locATT[j]  
   :val destroyedATT[j]  
   :val actATT[j]
   :val pulseLoc   
   :val actEnv   
   :val L
 \end{verbatim}
 }

The encoding the rest of the transitions follows the same approach, although one has to manually write all the transitions, which is a delicate and cumbersome task, prone to error (the running example, for the \async execution semantics, requires 26 transitions). Moreover, the transition commented above correspond to the most simple transition formula in Section~\ref{sec:mas-abs}, while other transitions require more verbose encodings in MCMT.

For instance, the transition formula (\ref{eq:Local1_async}) that represents the bulk update of the environment together with all agents which declared a local action is encoded a set of MCMT transitions, one for each possible local action of the environment (since disjunction is not allowed in the syntax of MCMT, hence multiple transitions must be specified for capturing disjunctions). In our example, the environment has two local actions, namely \texttt{pulseA} and \texttt{pulseB}, hence two MCMT transitions must be added, each containing the required if-then-else encoding that matches the case-defined functions of the form $\mathtt{case~of}$ used in Section~\ref{sec:mas-abs}. We list here only the first transition, for \texttt{pulseA}:

{\scriptsize
\begin{verbatim}
:transition
:var j
:guard 
	(= phase PL) 
	(= actE pulseA) 
	(= turn turnTEMPS)
:numcases 5

:case (= actATT[j] gotoA) 
:val A
:val destroyedATT[j]
:val NULL_Action
:val A
:val NULL_Action
:val P0
:val turnREST

:case (= actATT[j] gotoB) 
:val B
:val destroyedATT[j]
:val NULL_Action
:val A
:val NULL_Action
:val P0
:val turnREST

:case (= actATT[j] goTargetA) 
:val target
:val destroyedATT[j]
:val NULL_Action
:val A
:val NULL_Action
:val P0
:val turnREST

:case (= actATT[j] goTargetB) 
:val target
:val destroyedATT[j]
:val NULL_Action
:val A
:val NULL_Action
:val P0
:val turnREST

:case 
:val locATT[j]
:val destroyedATT[j]
:val NULL_Action
:val A
:val NULL_Action
:val P0
:val turnREST
 \end{verbatim}
 }

The guard of this transition requires the phase to be equal to $L$, the action of the environment to be \texttt{pulseA}. We do not comment here on turns, as this is a mere detail of the example at hand, but it is not required in general. 
There are 5 cases, which account for the four possible actions that each robot might have selected, plus a catch-all case that is used if no action was selected (namely the agent performed $\nop$). In each of these cases, the appropriate variable update is performed (again, the order is dictated by the order in which these variables were defined in the file -- see above).

This justifies the need of a user-oriented approach, which we comment in Section~\ref{sec:tool}. First, we very briefly report on how the solver is executed.

\subsubsection{Executing MCMT}

Once the textual encoding is done, MCMT can be simply executed via command line, specifying as argument the textual file: \texttt{./mcmt file.txt}. For more details and options, please refer to the MCMT manual \cite{MCMT-manual}. 
 
\subsection{\TOOLtitle}
\label{sec:tool}

In this section, we present and illustrate \TOOLtitle~\cite{SAFE}, i.e., our own implementation of a user interface that allows to directly employ in practice the results presented in this paper. The tool automatises the textual encoding of the \MASsystem as MCMT input files, by relying on a \emph{MAS-oriented} modelling framework which takes care of the textual conversion. This allows the user to focus on modelling the \MASsystem, i.e., the agent templates and the environment template, without worrying about how their constructs translate in the encoding for the model checker. As we show in this section, the tool also allows to convert the witnesses for unsafety that MCMT returns (when the input ABS is unsafe) back into executions of the original \MASsystem. The resulting tool chain constitutes a user-friendly and effective, implemented approach for modelling and verifying the safety of parameterized multi-agent systems. Some preliminary experimental evaluation is performed in Section~\ref{sec:evaluation}.

\subsubsection{The \TOOL modeling interface}
\label{sec:thegui}

The interface of \TOOLtitle is available at \cite{SAFE}, and can be used to model and encode into MCMT input files (i.e., into textual representations of array-based systems) any \MASsystem as those introduced in this paper. As an example, Figure~\ref{fig:tool_cannon} shows the running example of this paper as it appears in the \TOOL GUI.

\begin{figure}[p]
\centering{
\resizebox{\columnwidth}{!}{
\includegraphics{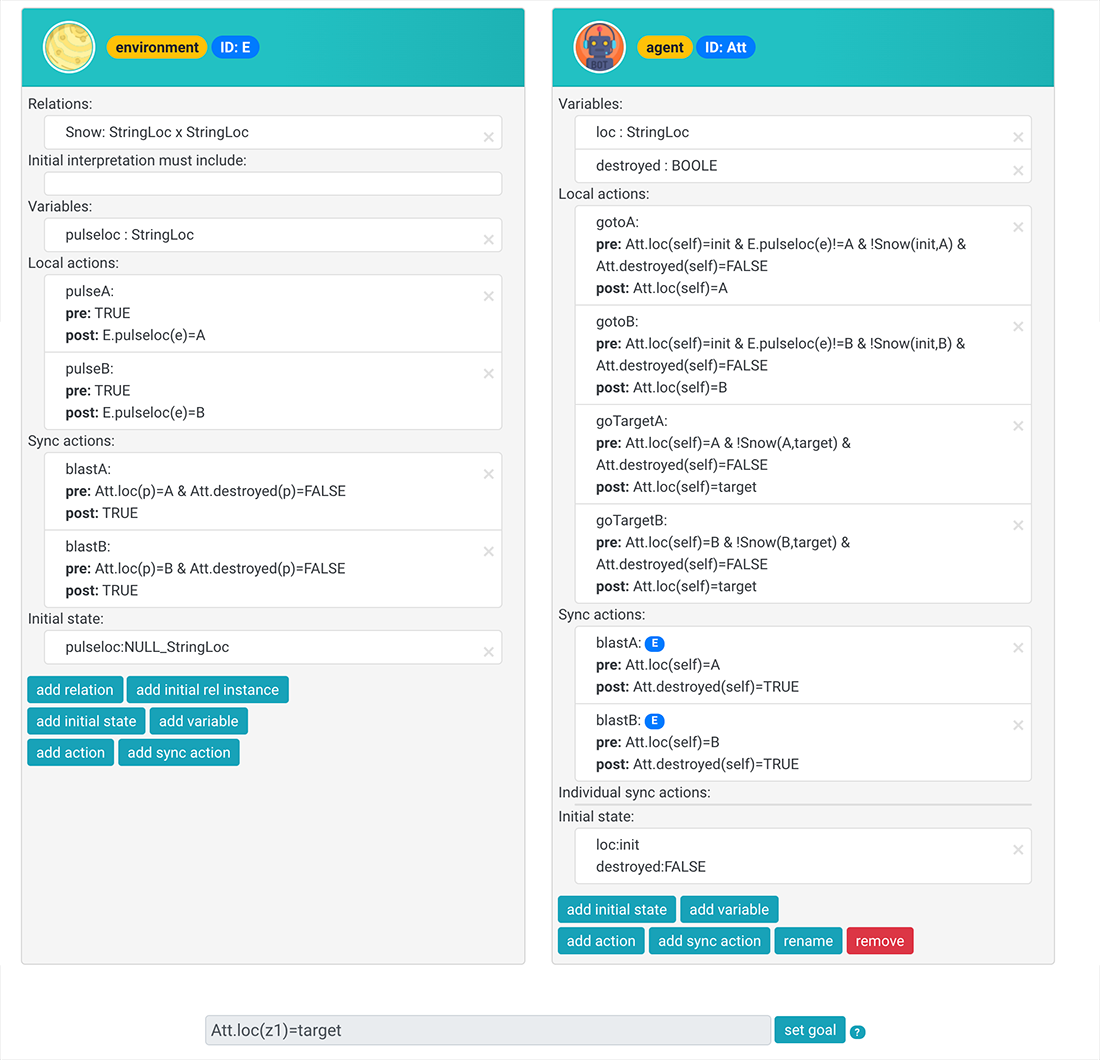}
}}
\caption{The main part of the \TOOL Base interface, showing the \MASsystem model for the running example. The GUI provides intuitive buttons to either add or remove any component from the templates, so that there is no manual coding required.}
\label{fig:tool_cannon}
\end{figure}

In its present version, there are a number of minor differences between \TOOL and the representations used here:

\begin{compactitem}
\item The version of \TOOL currently available, called \TOOL Base, can only be used to model \async \MASsystem{s}. We are currently working on \TOOL For All, an extension of \TOOL Base which admits the kind on universal quantification needed to represent \sync \MASsystem{s} (see the transition formulae in Section~\ref{sec:encoding_sync}, where universal quantification is used);
\item The representation of agent and environment templates used by \TOOL differs slightly from the one used here, although it is equivalent. Indeed, Definition~\ref{def:agent_template} defines an agent templates as a labelled, finite-state transition system, whereas \TOOL assumes a more succinct representation, where instead of listing explicitly the local states and transitions of the templates, we assume a STRIPS-like approach. Therefore, instead of an explicit finite-state machine labelled by actions, actions are specified by means of pre- and post- conditions; 
\item Similarly, the GUI of \TOOL Base does not allow to use disjunction in action preconditions. This implies that, whenever we want to model an action label $a$ that is available under two distinct conditions (i.e., a precondition of the form $\phi_1\lor\phi_2$, we need instead to use a distinct copy $a'$ of the action and assign to these the preconditions $\phi_1$ and $\phi_2$, respectively. Future versions of \TOOL will avoid this restriction. 
\end{compactitem}
 
 \smallskip
 At the same time, \TOOL Base includes some minor features that were not included in the formal model discussed here, as these are primarily implementation details: 

\smallskip
\begin{compactitem}
\item Apart from local actions and synchronization actions, the tool allows a further type of actions, called \emph{individual} synchronization actions, that can only be executed by the environment and by exactly another agent at the same time. In Section~\ref{sec:further_example} we provide an additional example of \MASsystem in which such feature is essential (note that this corresponds to a type of action considered in \cite{KouvarosAIJ16}, as discussed in Section~\ref{sec:related_pmas});
\item The \emph{turn-based} mechanics, which allows to alternate the actions of two distinct subsets of the templates (as in our running example), is a core-feature of the tool and can be easily enabled or disabled without the need of manually implement the alternation logic. In order to make this compatible with the execution semantics for synchronization actions, \TOOL Base allows an additional annotation of these actions, by which it is possible to specify the \emph{initiator} template of each synchronization action; 
\item The GUI allows to request additional conditions on the instances of \MASsystem that the verification should consider. For instance, we  can request that a certain template is not empty (that is, that at least a concrete agent with that template exists). More of these features are continuously being implemented in \TOOL.
\end{compactitem}


\section{Experimental evaluation}
\label{sec:evaluation}

In this section we report on preliminary, simple experiments run with MCMT on the inputs obtained as outputs of \TOOLtitle (Base). We first consider the running example of this paper, then introduce a further one from the literature.  

\subsection{Experiments on the running example}

The textual encoding of the ABS corresponding to the \MASsystem in the running example, obtained through \TOOL Base, is solved by MCMT v.2.9 (which uses an old version of Yices - the 1.0.40- as a background SMT solver) in 2.7 seconds on a early-2015 Macbook (2.7 GHz Dual-Core Intel Core i5, 8 GB RAM). MCMT correctly reports that the system is unsafe.\footnote{This example is available at: \url{http://safeswarms.club/page/mcmt/cannon}.}   

The input file, of which some parts are reported and commented in the previous section, has 3 local variables for the robot template $att$, 3 global variables and 26 transitions. 

MCMT gives in output the following sequence as witness of unsafety:

{
\begin{verbatim}
[t2][t17][t3_1][t15][t1][t16][t5_1][t15]
\end{verbatim}} 

\noindent
where each \texttt{t}$_n$ represents the execution of the $n$-th transition in the input file, following the order in which they appear. \TOOL provides an input dialogue in which one can paste such witnesses, so as to retrieve the ``run template" of the \MASsystem that corresponds to such sequences. We say that such run is a ``run template" because, unlike the definition of run of \MASsystem{s} given in Section~\ref{sec:runs}, it does not list the precise number of concrete agents that performed the action associated to that transitions. For the sequence above, the run template shown by \TOOL imposes the following sequence of actions:

{
\begin{verbatim}
pulseB, gotoA, pulseA, goTarget
\end{verbatim}} 

\noindent
in which, trivially, a number of agents reaches the target while avoiding the EMP pulse used by the target, which in this instance does not even attempt to use the blasts to destroy robots. In fact, transition 2 corresponds to equation (\ref{eq:Local0_async}) when $\act=\texttt{pulseB}$ and $t = e$, followed by transition 17 which corresponds to the bulk update as in equation (\ref{eq:Local1_async}). Similarly, transition 3 corresponds to (\ref{eq:Local0_async}) when $\act=\texttt{gotoA}$  and $t = Att$, followed again by a transition conforming to equation (\ref{eq:Local1_async}), and so on. 
 
When a \MASsystem is determined to be unsafe, \TOOL does not provide support for embedding into the model the witness provided by MCMT, so it is the responsibility of the user to update their MAS by taking insights from the witness and then check the new model again.  
 
By increasing the number of agent templates and number of waypoints required to reach the target on either of the two paths (i.e., by having waypoints $\texttt{A}_1,\ldots,\texttt{A}_n$ and $\texttt{B}_1,\ldots,\texttt{B}_n$), we can test the scalability of our approach (i.e., the use of \TOOL and MCMT for checking safety of \MASsystem{s}) with respect to the \emph{length} of possible runs. In these versions of the running example, cannon blasts hit all locations on the same path simultaneously, the EMP pulse can block all robots on the entire path at which it is directed, and the protocol of the cannon is so that the cannon can freely fire at either path without checking that there are available robot targets. This is achieved by adding a further \texttt{path} variable to robot templates.\footnote{These examples are available at: \url{http://safeswarms.club/page/mcmt/exZb} with \texttt{Z}=2..6}

\bigskip

\pgfplotsset{compat=1.14} 
\begin{tikzpicture}
\begin{axis}[
    width=12cm,
    height=6cm,
    ybar,
    bar width=0.3,
    ymin=0,
    ytick={3,22,49,96,176},
    xticklabel style={rotate=45,font=\footnotesize},
    xticklabels from table={data1.dat}{x}, 
    xtick=data, 
    enlarge x limits=0.25, 
    xlabel=\# of agent templates - \# of waypoints per path - \# of transitions,
    ylabel=seconds
    ]
    \addplot table[x expr=\coordindex,y=y] {data1.dat}; 
\end{axis}
\end{tikzpicture}

As it can be seen from the experiments, the number of transitions increase slightly, and the execution time increases as well as the number of actions and length of the shortest unsafe runs increases. Indeed, there is of course a direct relationship between the length of the \MASsystem runs that must be checked and the \emph{depth} of the  state-space exploration of MCMT. The number of variables remains instead the same (although more constants are introduced). These findings are not surprising, because the number of possible actions combinations, for agents and environment, remains the same at each step. 

If instead we increase the number of templates, by introducing further \emph{copies} of the robot template from the original problem instance (preserving only waypoints \texttt{A} and \texttt{B}), then both the number of transitions and variables increase substantially, leading to much longer execution times. This was expected, as the possible combinations of actions, when considering all agent templates, grow exponentially. We report below on the average execution times for the case of 1, 2, 3 and 4 robot templates (in addition to the environment template).\footnote{These examples are available at: \url{http://safeswarms.club/page/mcmt/exZ} with \texttt{Z}=2..4}

\bigskip

\pgfplotsset{compat=1.14} 
\begin{tikzpicture}
\begin{axis}[
    width=12cm,
    height=6cm,
    ybar,
    bar width=0.3,
    ymin=0,
    ytick={48,367,1530},
    xticklabel style={rotate=45,font=\footnotesize},
    xticklabels from table={data1.dat}{x}, 
    xtick=data, 
    enlarge x limits=0.25, 
    xlabel=\# of agent templates - \# of variables - \# of transitions,
    ylabel=seconds
    ]
    \addplot table[x expr=\coordindex,y=y] {data1.dat}; 
\end{axis}
\end{tikzpicture}

To run these experiments, we disabled the default setting of MCMT which puts a bound of 50 to the possible number of transitions. 

Interestingly, if we remove the assumption that the agents and the environment alternate their moves (thus also removing the global variable \texttt{turn} that is automatically added by \TOOL), determining the unsafety of the \MASsystem becomes much easier.\footnote{To replicate these experiments, it is sufficient to turn off the template alternation through the GUI.}

\bigskip

\pgfplotsset{compat=1.14} 
\begin{tikzpicture}
\begin{axis}[
    width=12cm,
    height=6cm,
    ybar,
    bar width=0.3,
    ymin=0,
    ytick={1,8,32,90},
    xticklabel style={rotate=45,font=\footnotesize},
    xticklabels from table={data1.dat}{x}, 
    xtick=data, 
    enlarge x limits=0.25, 
    xlabel=\# of agent templates - \# of variables - \# of transitions,
    ylabel=seconds
    ]
    \addplot table[x expr=\coordindex,y=y] {data1.dat}; 
\end{axis}
\end{tikzpicture}

Apparently, and contrary to the intuitive understanding, the verification becomes more challenging as the \MASsystem becomes more specified (that is, with more detailed pre- and post-conditions of actions, turn indicators, goal conjuncts). 

This highlights the limitations of this implementation, and the use of MCMT for larger examples, although the examples we can find in the literature are smaller than these, and involve typically two agent templates. Nonetheless, we should keep in mind that these problem instances are intrinsically computationally demanding, since checking the safety of a given \MASsystem is exponential in the number of variables (as these determine the number of global states) and the number of global transitions. For example, for the case of 3 agent templates, more than 800k calls were made to the SMT solver (i.e., to Yices) and 2k symbolic nodes were explored. For the case of 4 templates, 4M calls are made.    

\subsection{A further example}
\label{sec:further_example}

In this section we introduce a further example, inspired to the train-gate-controller scenario that has been repeatedly used in the literature. 

This scenarios allows us to discuss one of the features of \TOOL that are not included in this formal framework: \emph{individual synchronization actions} first discussed in Section~\ref{sec:thegui}. These are like regular synchronization actions, with the difference that exactly one agent synchronizes with the environment. 

\begin{example}
The \MASsystem is constituted by a small rail network where a (unbounded) number of trains share one critical section, namely a tunnel passage. The tunnel is equipped with a controller module that is responsible to regulate the access, so that at all times at most one train is in the tunnel. For doing so, both sides of the tunnel are equipped with traffic lights. While the status of the traffic lights can be either red or green, each train can be in four different states: queuing for the permission to enter, in the tunnel, away from the tunnel, forbidden to enter. Moreover, trains can be either normal trains or prioritized trains, so that while a prioritized train can enter the tunnel at any given time (if the tunnel is free), normal trains can only be allowed to enter (by the controller) when there are no priority trains waiting. To accomplish this, the traffic lights has two distinct variants of the green color: prioritized green and normal green. Prioritized green is used by the controller to serve prioritized trains, whereas normal green is used by the controller to serve normal trains. As a consequence, prioritized trains can be forbidden to enter only if away from the tunnel (by changing the color of the light to normal green).

The scenario can be encoded as a \MASsystem composed of an agent template representing prioritized trains, an agent template representing normal trains, and an environment template representing the controller. A prioritized train is initially in state \texttt{Waiting}, the controller is initially in state \texttt{PGreen}, and a normal train is initially in state \texttt{Forbidden}. Therefore prioritised trains are initially waiting to enter the tunnel, normal trains are initially locked from entering the tunnel, and the controller initially serves only prioritized trains. The actions \texttt{p\_enter} and \texttt{p\_exit} are individual synchronization actions modelling (exactly one of the) prioritized trains entering and exiting the tunnel. Similarly, the actions \texttt{n\_enter} and \texttt{n\_exit} are individual synchronization actions enabling the normal trains to enter and exit the tunnel. The action \texttt{allow\_n} is a synchronous action representing the action of the controller to allow normal train (switch to normal green), and similarly \texttt{allow\_p}. These need to be synchronization actions (rather than local actions of the environment) since their execution must be \emph{visible} to the trains as well, which also change their internal local state. Finally, the actions \texttt{n\_approach} and \texttt{p\_approach} are local actions that trains execute for approaching the tunnel and hence wait for their green light.

The safety checking task is to verify that it is not possible for two trains (of any type) to be in the tunnel at the same time, and it is easy to verify, given the description above, that the system is in fact safe. Indeed, the street light guarantees (via individual synchronization actions) that only one train enters the tunnel; moreover, the color is set again to green (either normal or priority) only when a train exits the tunnel. 
\end{example}

As expected, MCMT reports that the \MASsystem encoding this scenario is safe. The verification takes 6.7 seconds on the same setup as in the previous experiments.\footnote{This example is available at: \url{http://safeswarms.club/page/mcmt/train}.}

\begin{figure}[p]
\centering{
\resizebox{\columnwidth}{!}{
\includegraphics{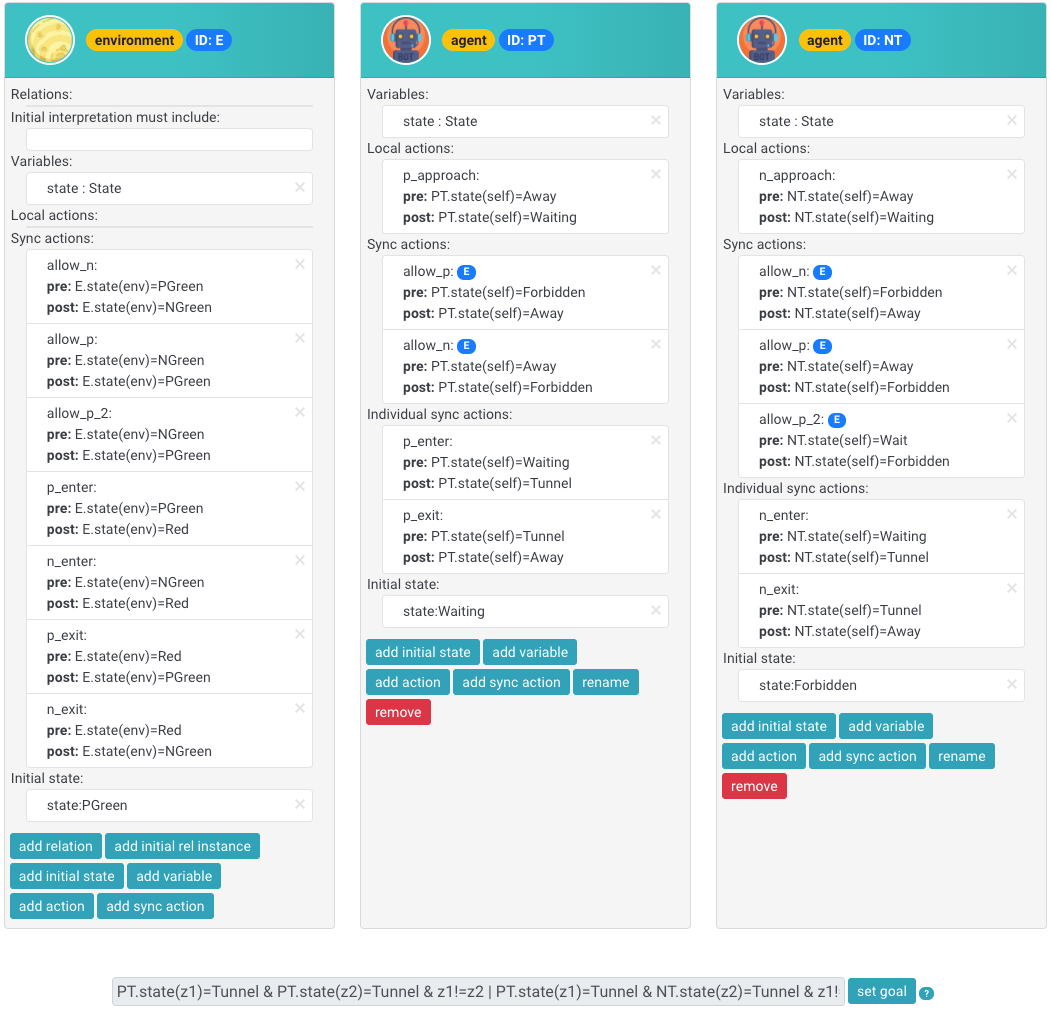} 
}}
\caption{The main part of the \TOOL Base interface, showing the \MASsystem model for the trains example.}
\label{fig:tool_trains}
\end{figure}

As before, we attempt to stress the use of MCMT for the verification of this scenario by adding further copies of agent templates (in this case, normal trains), to make the number of variables and global transitions grow. Note that the safety formula grows (in length) as the number of templates grows as we want to forbid any combination of trains in the tunnel. The execution of MCMT for 4 normal train templates (with 11 variables and 58 transitions) reports an average execution time of 58 seconds. 
%
%
To encode the case of 6 normal train templates, the formula would need 27 disjuncts. Due to the current limitation of MCMT, a bound of 9 exists, so we cannot increase further the number of templates.

As a final note, observe that a \sync \MASsystem allows us to specify, by design, that no train can enter the tunnel if there exists another train currently there (or, alternatively, that the action of turning the street light to color green cannot be executed). Indeed, universal quantification allows us to write, as action preconditions, that all trains are not in the tunnel. This enhanced input formalism should allow a more natural modelling of examples. To implement this in practice in \TOOL, however, we need the upcoming version \TOOL For All.


\section{Conclusions and Future Work}
\label{sec:futurework}

We devised a novel technique for checking safety of \MASsystem{s}. The advantages are various, as summarised in Sec.\ref{sec:related}. 
This opens up towards richer parameterised and data-aware MAS settings to be tackled in a novel, solid, vast and extensively studied range of formal techniques which were so far not directly available to the AI verification community. 
In future work, we will enrich this \myi beyond \async \MASsystem{s}, \myii  with full-fledged, relational, read/write databases, \myiii also add cardinality constraints and \myiv experimentally evaluate and compare MCMT as a model checker for this problem. 
The possibility of continuing along so different extensions within the same framework is a direct consequence of this work, which establishes a formal connection between parameterised verification in MAS and the long-standing tradition of SMT-based model checking for array-based systems.


\bibliographystyle{plain}
\bibliography{ecai}

\begin{thebibliography}{10}

\bibitem{MCMT-manual}
{MCMT: Model Checker Modulo Theories}.
\newblock http://users.mat.unimi.it/users/ghilardi/mcmt.
\newblock Accessed: 2020-08-01.

\bibitem{SAFE}
{SAFE: the Swarm Safety Detector}.
\newblock http://www.safeswarms.club.
\newblock Accessed: 2020-08-01.

\bibitem{Abdulla96}
Parosh~Aziz Abdulla, Karlis Cerans, Bengt Jonsson, and Yih-Kuen Tsay.
\newblock General decidability theorems for infinite-state systems.
\newblock In {\em Proceedings 11th Annual IEEE Symposium on Logic in Computer
  Science}, pages 313--321, July 1996.

\bibitem{lossy}
Parosh~Aziz Abdulla and Bengt Jonsson.
\newblock Verifying programs with unreliable channels.
\newblock In {\em Proc. LICS '93-8th IEEE Int. Symp. on Logic in Computer
  Science}, pages 160--170, 1993.

\bibitem{fmsd}
Francesco Alberti, Roberto Bruttomesso, Silvio Ghilardi, Silvio Ranise, and
  Natasha Sharygina.
\newblock An extension of lazy abstraction with interpolation for programs with
  arrays.
\newblock {\em Form. Methods Syst. Des.}, 45(1):63--109, 2014.

\bibitem{FI}
Francesco Alberti, Silvio Ghilardi, and Natasha Sharygina.
\newblock A framework for the verification of parameterized infinite-state
  systems.
\newblock {\em Fund. Inform.}, 150(1):1--24, 2017.

\bibitem{AlechinaBGFLV19}
Natasha Alechina, Tom{\'{a}}s Br{\'{a}}zdil, Giuseppe {De Giacomo}, Paolo
  Felli, Brian Logan, and Moshe~Y. Vardi.
\newblock Proceedings of the thirty-first {AAAI} conference on artificial
  intelligence.
\newblock pages 2646--2653, 2019.

\bibitem{BarrettT18}
Clark~W. Barrett and Cesare Tinelli.
\newblock Satisfiability modulo theories.
\newblock In {\em Handbook of Model Checking.}, pages 305--343. 2018.

\bibitem{BelardinelliIJCAI17}
Francesco Belardinelli, Panagiotis Kouvaros, and Alessio Lomuscio.
\newblock Parameterised verification of data-aware multi-agent systems.
\newblock In {\em Proceedings of the Twenty-Sixth International Joint
  Conference on Artificial Intelligence, {IJCAI} 2017}, pages 98--104, 2017.

\bibitem{Bloem15}
Roderick Bloem, Swen Jacobs, and Ayrat Khalimov.
\newblock {\em Decidability of Parameterized Verification}.
\newblock Morgan \& Claypool Publishers, 2015.

\bibitem{BullingGJ15}
Nils Bulling, Valentin Goranko, and Wojciech Jamroga.
\newblock Logics for reasoning about strategic abilities in multi-player games.
\newblock In {\em Models of Strategic Reasoning - Logics, Games, and
  Communities}, pages 93--136. 2015.

\bibitem{CGGMR18}
Diego Calvanese, Silvio Ghilardi, Alessandro Gianola, Marco Montali, and Andrey
  Rivkin.
\newblock Verification of data-aware processes via array-based systems
  (extended version).
\newblock Technical Report arXiv:1806.11459, arXiv.org, 2018.

\bibitem{IJCAR20-ext}
Diego Calvanese, Silvio Ghilardi, Alessandro Gianola, Marco Montali, and Andrey
  Rivkin.
\newblock {Combined Covers and Beth Definability (Extended Version)}.
\newblock Technical Report arXiv:1911.07774, arXiv.org, 2019.

\bibitem{BPM19}
Diego Calvanese, Silvio Ghilardi, Alessandro Gianola, Marco Montali, and Andrey
  Rivkin.
\newblock Formal modeling and {SMT}-based parameterized verification of
  data-aware {BPMN}.
\newblock In {\em Proc.\ of {BPM}}, volume 11675 of {\em LNCS}. Springer, 2019.

\bibitem{CGGMR19}
Diego Calvanese, Silvio Ghilardi, Alessandro Gianola, Marco Montali, and Andrey
  Rivkin.
\newblock From model completeness to verification of data aware processes.
\newblock In {\em Description Logic, Theory Combination, and All That}, volume
  11560 of {\em LNCS}. Springer, 2019.

\bibitem{cade19}
Diego Calvanese, Silvio Ghilardi, Alessandro Gianola, Marco Montali, and Andrey
  Rivkin.
\newblock Model completeness, covers and superposition.
\newblock In {\em Proc.\ of {CADE}}, volume 11716 of {\em LNCS}. Springer,
  2019.

\bibitem{ARCADE}
Diego Calvanese, Silvio Ghilardi, Alessandro Gianola, Marco Montali, and Andrey
  Rivkin.
\newblock Verification of data-aware processes: Challenges and opportunities
  for automated reasoning.
\newblock In {\em Proc. of ARCADE}, volume 311. EPTCS, 2019.

\bibitem{IJCAR20}
Diego Calvanese, Silvio Ghilardi, Alessandro Gianola, Marco Montali, and Andrey
  Rivkin.
\newblock {Combined Covers and Beth Definability}.
\newblock In {\em Proc. of {IJCAR}}, volume 12166 of {\em LNCS}, pages
  181--200. Springer, 2020.

\bibitem{MSCS20}
Diego Calvanese, Silvio Ghilardi, Alessandro Gianola, Marco Montali, and Andrey
  Rivkin.
\newblock {SMT}-based verification of data-aware processes: a model-theoretic
  approach.
\newblock {\em Mathematical Structures in Computer Science}, 30(3):271--313,
  2020.

\bibitem{Clarke2018}
Edmund~M. Clarke, Thomas~A. Henzinger, Helmut Veith, and Roderick Bloem,
  editors.
\newblock {\em Handbook of Model Checking}.
\newblock Springer, 2018.

\bibitem{ConduracheMG19}
Rodica Condurache, Riccardo {De Masellis}, and Valentin Goranko.
\newblock Dynamic multi-agent systems: Conceptual framework, automata-based
  modelling and verification.
\newblock In {\em Proc. of {PRIMA} 2019: Principles and Practice of Multi-Agent
  Systems - 22nd International Conference}, pages 106--122, 2019.

\bibitem{Emerson03}
E.~Allen Emerson and Kedar~S. Namjoshi.
\newblock On reasoning about rings.
\newblock {\em Int. J. Found. Comput. Sci.}, 14(4):527--550, 2003.

\bibitem{EsparzaGLM17}
Javier Esparza, Pierre Ganty, J{\'{e}}r{\^{o}}me Leroux, and Rupak Majumdar.
\newblock Verification of population protocols.
\newblock {\em Acta Inf.}, 54(2):191--215, 2017.

\bibitem{Finkel01}
Alain Finkel and Philippe Schnoebelen.
\newblock Well-structured transition systems everywhere!
\newblock {\em Theoretical Computer Science}, 256(1):63 -- 92, 2001.
\newblock ISS.

\bibitem{BPM20}
Silvio Ghilardi, Alessandro Gianola, Marco Montali, and Andrey Rivkin.
\newblock Petri nets with parameterised data: Modelling and verification.
\newblock In {\em Proc. of {BPM}}, LNCS. Springer, 2020.

\bibitem{ijcar08}
Silvio Ghilardi, Enrica Nicolini, Silvio Ranise, and Daniele Zucchelli.
\newblock Towards {SMT} model checking of array-based systems.
\newblock In {\em Proc. of {IJCAR}}, pages 67--82, 2008.

\bibitem{lmcs}
Silvio Ghilardi and Silvio Ranise.
\newblock Backward reachability of array-based systems by {SMT} solving:
  Termination and invariant synthesis.
\newblock {\em Log. Methods Comput. Sci.}, 6(4), 2010.

\bibitem{GhilardiR10}
Silvio Ghilardi and Silvio Ranise.
\newblock {MCMT:} {A} model checker modulo theories.
\newblock In {\em Automated Reasoning, 5th International Joint Conference,
  {IJCAR} 2010, Edinburgh, UK, July 16-19, 2010. Proceedings}, pages 22--29,
  2010.

\bibitem{GhilUniv}
Silvio Ghilardi, Silvio Ranise, and Thomas Valsecchi.
\newblock Light-weight smt-based model checking.
\newblock {\em Electron. Notes Theor. Comput. Sci.}, 250(2):85--102, 2009.

\bibitem{john2012counter}
Annu John, Igor Konnov, Ulrich Schmid, Helmut Veith, and Josef Widder.
\newblock Counter attack on byzantine generals: Parameterized model checking of
  fault-tolerant distributed algorithms, 2012.

\bibitem{KouvarosAIJ16}
Panagiotis Kouvaros and Alessio Lomuscio.
\newblock Parameterised verification for multi-agent systems.
\newblock {\em Artif. Intell.}, 234:152--189, 2016.

\bibitem{KouvarosAAAI17}
Panagiotis Kouvaros and Alessio Lomuscio.
\newblock Parameterised verification of infinite state multi-agent systems via
  predicate abstraction.
\newblock In {\em Proceedings of the Thirty-First {AAAI} Conference on
  Artificial Intelligence, February 4-9, 2017, San Francisco, California,
  {USA.}}, pages 3013--3020, 2017.

\bibitem{KouvarosLPP19}
Panagiotis Kouvaros, Alessio Lomuscio, Edoardo Pirovano, and Hashan Punchihewa.
\newblock Formal verification of open multi-agent systems.
\newblock In {\em Proceedings of the 18th International Conference on
  Autonomous Agents and MultiAgent Systems, {AAMAS} '19, 2019}, pages 179--187,
  2019.

\bibitem{Pnueli02}
Amir Pnueli, Jessie Xu, and Lenore Zuck.
\newblock Liveness with (0,1, infty)- counter abstraction.
\newblock In Ed~Brinksma and Kim~Guldstrand Larsen, editors, {\em Computer
  Aided Verification}, pages 107--122, Berlin, Heidelberg, 2002. Springer
  Berlin Heidelberg.

\end{thebibliography}

\appendix

\newpage


\section{Appendix}
\sectionmark{Appendix: Proof of Theorems}

In this Appendix we provide the full proofs of the theorems from Section~\ref{sec:results}.

\subsection{Proof of Theorem~\ref{thm:sound-complete}.}\label{sec:app1}

In this section we provide a proof of Theorem~\ref{thm:sound-complete}. The proof is quite involved and follows the structure, even if with several modifications, of an analogous theorem for Relational Artifact Systems (RASs \cite{MSCS20}). 

We observe  that transition formulae introduced in Section~\ref{sec:mas-abs} do not go beyond first-order logic. We clarify this point now, before going into formal proofs.

First of all, to obtain a more compact
representation, we made use there of definable extensions as a means for
introducing so-called \emph{case-defined functions}.  We fix a signature
$\Sigma$ and a $\Sigma$-theory $T$; a \emph{$T$-partition} is a finite set
$\kappa_1(\ux), \dots, \kappa_n(\ux)$ of quantifier-free formulae
such that $T\models \forall \ux \bigvee_{i=1}^n \kappa_i(\ux)$ and
$T\models \bigwedge_{i\not=j}\forall \ux \neg (\kappa_i(\ux)\wedge
\kappa_j(\ux))$.  Given such a $T$-partition
$\kappa_1(\ux), \dots, \kappa_n(\ux)$ together with $\Sigma$-terms
$t_1(\ux), \dots, t_n(\ux)$ (all of the same target sort), a
\emph{case-definable extension} is the $\Sigma'$-theory $T'$, where
$\Sigma'=\Sigma\cup\{F\}$, with $F$ a ``fresh'' function symbol (i.e.,
$F\not\in\Sigma$)\footnote{Arity and source/target sorts for $F$ can be
 deduced from the context (considering that everything is well-typed).}, and
$T'=T \cup\bigcup_{i=1}^n \{\forall\ux\; (\kappa_i(\ux) \to F(\ux) =
t_i(\ux))\}$.
Intuitively, $F$ represents a case-defined function, which can be reformulated
using nested if-then-else expressions and can be written as
$ F(\ux) ~:=~ \mathtt{case~of}~ \{\kappa_1(\ux):t_1;\cdots;\kappa_n(\ux):t_n\}.
$ By abuse of notation, we identify $T$ with any of its case-definable
extensions $T'$.  In fact, it is easy to produce from a $\Sigma'$-formula
$\phi'$ a $\Sigma$-formula $\phi$ equivalent to $\phi'$ in all models of $T'$:
just remove (in the appropriate order) every occurrence $F(\uv)$ of the new
symbol $F$ in an atomic formula $A$, by replacing $A$ with
$\bigvee_{i=1}^n (\kappa_i(\uv) \land A(t_i(\uv)))$.

Secondly, he lambda-abstraction definitions in~\eqref{eq:trans1} will make the proof of Lemma~\ref{lem:eq1} below smooth.
Recall that an expression like
\[
b = \lambda y. F(y,\uz)
\]
can be seen as a mere abbreviation of $\forall y~b(y)=F(y,\uz)$.
However, the use of such abbreviation makes clear that, e.g., a formula like
\[
\exists b~( b = \lambda y. F(y,\uz) \wedge \phi(\uz, b))
\]
is equivalent to
\begin{equation}\label{eq:aux}
\phi(\uz, \lambda y. F(y,\uz)/b)~~.
\end{equation}
Since our $\phi(\uz, b)$ is in fact a first-order formula, our $b$ can occur in it only in terms like $b(t)$, so that in~\eqref{eq:aux} all
occurrences of $\lambda$ can be eliminated by the so-called $\beta$-conversion: replace $\lambda y F(y,\uz)(t)$ by $F(t, \uz)$. Thus, in the end, either we use definable extensions or definitions via lambda abstractions, \emph{the formulae we manipulate can always be converted into plain first-order formulae}.


\begin{lemma}\label{lem:eq1}
	The preimage of a state formula is logically equivalent to a state formula.
\end{lemma}
\begin{proof} We manipulate the formula
	\begin{equation}\label{eq:pre}
	\exists \ux'\,\exists \ua'\, (\tau(\ux, \ua, \ux', \ua') \wedge \exists \ue~ \phi(\ue, \ux',\ua'))
	\end{equation}
	%
	up to logical equivalence, where $\tau$ is given by\footnote{Actually, $\tau$ is a disjunction of such formulae, but it easily seen that disjunction can be
		accommodated by moving existential quantifiers back-and-forth through them.}
	\begin{equation}\label{eq:trans2}
	\exists \ue_0\left(\gamma(\ue_0, \ux, \ua) \wedge
	\ux'= \uF(\ue_0, \ux, \ua) \wedge \ua'=\lambda y. \uG(y,\ue_0,\ux, \ua)
	\right)
	\end{equation}
	where $\ue, \ue_0$ are variables of index type  (here we used plain equality  for conjunctions of equalities, e.g. $\ux'= \uF(\ue_0, \ux, \ua)$ stands for $\bigwedge_i x'_i= F_i(\ue, \ux, \ua)$).
	Repeated substitutions  show that~\eqref{eq:pre} is equivalent to
	\begin{equation}
	\exists \ue\,\exists \ue_0\, \left(\gamma(\ue_0, \ux, \ua) \wedge \phi(\ue, \uF(\ue_0, \ux,\ua)/\ux',\lambda y.\uG(y,\ue_0,\ux, \ua)/\ua' )\right)
	\end{equation}
	which is a state formula.
\end{proof}
\vskip 2mm

We underline that Lemmas~\ref{lem:eq1} gives an explicit effective procedure
for computing preimages. At this point, the only formulae we need to test for satisfiability in lines 2 and 3 of  the backward reachability algorithm are
the $\exists\forall$-formulae introduced below.

Let us call $\exists\forall$-formulae the formulae of the kind 
\begin{equation}\label{eq:s02}
\exists \ue\; \forall \ui \; \phi(\ue, \ui, \ux, \ua)
\end{equation}
where the variables $\ue, \ui$ are variables whose sort is of index type and $\phi$ is quantifier-free.
The crucial point for the following lemma to hold is that the \emph{universally} quantified variables in $\exists\forall$-formulae are all of index type:

\begin{lemma}\label{lem:sat}  
	The satisfiability of a $\exists\forall$-formula in a model of an AB-PMAS is decidable. 
\end{lemma}
\begin{proof}
	First of all, notice that a $\exists\forall$-formula~\eqref{eq:s02} is equivalent to a disjunction of formulae of the kind
	\begin{equation}\label{eq:s03}
	\exists \ue\; ({\rm AllDiff}(\ue) \wedge \forall \ui \; \phi(\ue, \ui, \ux, \ua))
	\end{equation}
	where ${\rm AllDiff}(\ue)$ says that any two variables of the same sort from the $\ue$ are distinct (to this aim, it is sufficient to guess a partition and to keep, via a substitution,
	only one element for each equivalence class).\footnote{In the MCMT implementation, state formulae are always maintained so that all existential variables occurring in  them are differentiated and there is no need of this expensive computation step.} So we can freely assume that $\exists\forall$-formulae are all of the kind~\eqref{eq:s03}.
	

	Let us consider now the set of all (sort-matching) substitutions $\sigma$ mapping the $\ui$ to the $\ue$.
	The formula~\eqref{eq:s03} is satisfiable in a model of the AB-PMAS iff so it is the formula
	\begin{equation}\label{eq:s04}
	\exists \ue\; ({\rm AllDiff}(\ue) \wedge \bigwedge_{\sigma}\phi(\ue, \ui\sigma, \ux, \ua))
	\end{equation}
	(here $\ui\sigma$ means the componentwise application of $\sigma$ to the $\ui$): this is because, if~\eqref{eq:s04} is satisfiable in a structure $\cM$, then we can take as $\cM'$ the
	same $\Sigma$-structure as $\cM$, but with the interpretation of the index sorts restricted only to the elements named by the $\ue$ and get in this way that $\cM'$ satisfies ~\eqref{eq:s03}.	Thus, we can freely concentrate on the satisfiability problem of formulae of the kind~\eqref{eq:s04} only.
	
	Now, the only atoms occurring in the subformula  $\phi(\ue, \ui\sigma, \ux, \ua))$ of~\eqref{eq:s04} whose argument terms are terms of index sorts are of the kind $e_s=e_j$, so all such atoms can be 
	replaced either by $\top$ or by $\bot$ (depending on whether we have $s=j$ or not). So we can assume that there are no such atoms in $\phi(\ue, \ui\sigma, \ux, \ua))$ and as a result, the variables $\ue$ can only occur there as arguments of the $\ua$.

	Let now $\ua[\ue]$ be the tuple of the terms among the terms of the kind $a_j[e_s]$ which are well-typed. Since in~\eqref{eq:s04} the $\ue$ can only occur as arguments of array  varables, as observed above, the formula~\eqref{eq:s04} is in fact of the kind
	\begin{equation}\label{eq:s05}
	\exists \ue\; ({\rm AllDiff}(\ue) \wedge \psi(\ux, \ua[\ue]/\uz))
	\end{equation}
	where $\psi(\ux,\uz)$ is a quantifier-free $\Sigma$-formula and $\psi(\ux, \ua[\ue]/\uz)$ is obtained from  $\psi$ by replacing the variables $\uz$ by the terms $\ua[\ue]$
	(notice that the $\uz$ are of element sorts).
	
	It is now evident that~\eqref{eq:s05} is satisfiable  in a model of the AB-PMAS iff the formula
	\begin{equation}\label{eq:s06}
	\psi(\ux,\uz)
	\end{equation}
	is satisfiable in $\Sigma$-structure  $\cM$. In fact, if we are given a $\Sigma$-structure
	$\cM$ and an assignment satisfying~\eqref{eq:s06}, we can easily expand $\cM$  to a model of the AB-PMAS by taking the $e$'s themselves as the elements of the interpretation of the index sorts;  in the so-expanded structure, we can interpret the array variables $\ua$ by taking  the $\ua[\ue]$  to be the elements assigned to the $\uz$ 
	in the satisfying assignment for~\eqref{eq:s06}.
	
	The satisfiability of~\eqref{eq:s06}  in a $\Sigma$-structure $\cM$ is clearly decidable (since the constraint satisfiability problem for $EUF$ is decidable).
\end{proof}
\vskip 2mm
The instantiation algorithm of Lemma~\ref{lem:sat} can be used to discharge the satisfiability tests in lines 2 and 3 of the algorithm described in Figure~\ref{fig:algorithm}, because the conjunction of a state formula and of the negation of a state formula is a $\exists\forall$-formula (notice that
$\iota$ is itself the negation of a state formula, according to the format of the initial formula).

\vskip 2mm\noindent
\textbf{Theorem~\ref{thm:sound-complete}}~\emph{Backward search for the safety problem is \emph{sound} and \emph{complete} 
for \emph{\async} \ABMAS{s}.} 
\vskip 2mm
\begin{proof}
	Recall that an \ABMAS $\ABM$ is safe w.r.t $\upsilon$ iff there is no interpretation $\I_0$ of relations, no $k\geq 0$ and no possible assignment to the individual and array variables $\ux^0,\ua^0, \ldots, \ux^k, \ua^k$ such that the formula~\eqref{eq:smc1}
	\[
	\iota(\ux^0, \ua^0) \wedge \tau(\ux^0,\ua^0, \ux^1, \ua^1) \wedge \cdots \wedge\tau(\ux^{k-1},\ua^{k-1},
	\ux^k,\ua^{k})\wedge \upsilon(\ux^k,\ua^{k})
	\]
	is valid in any model $\cN$ of $\ABM$.

	Hence, $\ABM$ is unsafe iff for some $n$, the formula~\eqref{eq:smc1} is satisfiable in a $\cN$ of $\ABM$. Thus, we shall concentrate on satisfiability in models of $\ABM$ in order to prove the Theorem.

	First of all, let us call $B_n$ (resp. $\phi_n$), with $n\geq 0$, the status of the variable $B$ (resp. $\phi$) after $n$ executions in Line 4 (resp. Line 5) of the backward search algorithm ($n=0$ corresponds to the status of the variables in Line~1). 
	Notice also that we have
	\begin{equation}\label{eq:pre_phi}
	T\models \phi_{j+1}\leftrightarrow Pre(\tau, \phi_j)
	\end{equation} for all $j$ and that
	\begin{equation}\label{eq:invariant}
	T\models B_n \leftrightarrow \bigvee_{0\leq j<n} \phi_j
	\end{equation}
	is an invariant of the algorithm.

	Now, the formula~\eqref{eq:smc1} is satisfiable in a model $\cN$ under a suitable assignment iff the formula
	\begin{eqnarray*}
		\iota(\ux^0, \ua^0) ~~\wedge
		& \exists \ua^1\exists\ux^1 (\tau(\ux^0,\ua^0, \ux^1, \ua^1) \wedge \cdots ~~~~~~~~~~~~~~~~~~~~~~~~~~~~~~~~~~~~~~~~
		\\ &
		\cdots
		\wedge \exists \ua^k\exists \ux^k(\tau(\ux^{k-1},\ua^{k-1},
		\ux^k,\ua^{k})\wedge \upsilon(\ux^k,\ua^{k}))\cdots)
	\end{eqnarray*}
	is satisfiable in $\cN$ under a suitable assignment; by Lemma~\ref{lem:eq1}, the latter is equivalent to a formula of the kind
	\begin{equation}\label{eq:a}
	\iota(\ux, \ua)~\wedge~\exists \ue\,\phi(\ue,\ux, \ua)
	\end{equation}
	where $\exists \ue\,\phi(\ue,\ux, \ua)$ is a state formula (thus $\phi$ is quantifier-free and the $\ue$
	are variables of index sorts  - we renamed $\ux^0, \ua^0$ as $\ux, \ua$).
	However the satisfiability of~\eqref{eq:a} is the same as the satisfiability of
	$\exists \ue\,(\iota(\ux,\ua) \wedge\phi(\ue,\ux, \ua))$; 
	the latter, in view of the structure of the initial formula, is
	a $\exists\forall$-formula and so
	Lemma~\ref{lem:sat} applies and shows that its satisfiability in a model of the AB-PMAS is decidable. The satisfiability of~\eqref{eq:smc1} is clearly equivalent to the satisfiability of $\iota\land\phi_n$: hence, the satisfiability test in Line~3 is effective thanks to the decidability guaranteed by Lemma~\ref{lem:sat}. In addition, if the backward search algorithm terminates with an $\mathsf{unsafe}$ outcome, then in view of Line~3 we clearly have that $\ABM$ is really unsafe.

	Now consider the satisfiability test in Line~2. This is again a satisfiability test for a $\exists\forall$-formula, thus it is decidable because of Lemma~\ref{lem:sat}. In case of a $\mathsf{safe}$ outcome, we have that $T\models \phi_n\to B_n$; 
	we claim that, if we continued executing the loop of backward search algorithm, we would nevertheless get that:
\begin{equation}\label{eq:claim}
T\models B_m\leftrightarrow B_n
\end{equation}
for all $m\geq n$.
We justify Claim~\eqref{eq:claim} below. 

From $T\models \phi_n\to B_n$, taking into  consideration that Formula~\eqref{eq:pre_phi} 
holds,
we get $T\models\phi_{n+1} \to Pre(\tau,B_n)$. Since $Pre$ commutes with disjunctions (i.e., $Pre(\tau,\bigvee_j \phi_j)$
is logically equivalent to $\bigvee_j Pre(\tau,\phi_j)$),  we also have $T\models Pre(\tau,B_n)\leftrightarrow \bigvee_{1\leq j\leq n} \phi_j$ by the Invariant~\eqref{eq:invariant} and
by Formula~\eqref{eq:pre_phi}
again.  
By using the entailment $T\models \phi_n\to B_n$ once more,  we get  $T\models\phi_{n+1} \to B_n$ and also that $T\models B_{n+1}\leftrightarrow B_n$, thus we finally obtain that $T\models \phi_{n+1} \to B_{n+1}$.  This argument can be repeated for all $m\geq n$, obtaining that $T\models B_m\leftrightarrow B_n$ for all $m\geq n$, i.e. Claim~\eqref{eq:claim}.
		
	This would entail that $\iota\wedge \phi_m$ is always unsatisfiable (because of~\eqref{eq:invariant} and because $\iota\wedge \phi_j$ was unsatisfiable
	for all $j<n$), which is the same (as remarked above) as saying that all formulae~\eqref{eq:smc1} are unsatisfiable. Thus $\ABM$ is safe.
\end{proof}

To sum up, in this subsection we remarked that for the algorithm described in Figure~\ref{fig:algorithm}, to be effective, 
we need
decision procedures for discharging the satisfiability tests in Lines~2-3. We noticed that the only formulae we need to test in these lines
        have a specific form (i.e. 
        they are
        $\exists\forall$-formulae). 
        In 
        the first technical lemmas (Lemmas~\ref{lem:eq1}) we show that 
        the preimage of a state formula is again a state formula; 
then, in a second technical lemma (Lemma~\ref{lem:sat}), we show that entailments between state formulae 
(more generally, satisfiability of formulae of the kind $\exists\forall$) can
be decided via finite instantiation techniques. These observations
  make both safety and fixpoint tests effective
  and constitute the skeleton of the proof of Theorem~\ref{thm:sound-complete}.

\subsection{Proof of Theorem~\ref{thm:decid-local}}\label{sec:app2}
\label{sec:termination}

Theorem~\ref{thm:sound-complete} gives a semi-decision procedure for unsafety: if
the system is unsafe, the procedure discovers it, but if the system is safe,
the procedure (still correct) may not terminate.  Termination is much more
difficult to achieve. 
 We present a termination result for PMASs,  obtained
via the use of well quasi-orders: the syntactic condition that guarantees termination was introduced in \cite{MSCS20}, but the result we provide in this paper is novel, since it applies to a different class of array-based systems. 
The strategy for proving termination consists of isolating sufficient conditions that imply that the embeddability relation between
models is a well-quasi-ordering.

Before stating and proving Theorem~\ref{thm:decid-local}, we need to recall 
 some basic facts about well-quasi-orders.  Recall that a
\emph{well-quasi-order} (wqo) is a set $W$ endowed with a reflexive-transitive
relation $\leq$ having the following property: for every infinite succession
\[
w_0, w_1, \dots, w_i, \dots
\]
of elements from $W$ there are $i, j$ such that $i<j$ and $w_i\leq w_j$
The fundamental result about wqo's that we are going to use in this section is the following theorem, known as Dickson's Lemma:   

\begin{theorem}
	The cartesian product of $k$-copies of $\mathbb N$ (and also of $\mathbb N \cup \{\infty\}$), with componentwise ordering, is a wqo.
\end{theorem}

Let $\tilde\Sigma$ be $\Sigma\cup \{\ua, \ux\}$, that is, the relational signature $\Sigma$ from Section~\ref{sec:mas-abs}
expanded with unary function symbols $\ua$ (one for every array variable $a$) and constants $\ux$ (thus, a
$\tilde \Sigma$-structure is a $\Sigma$-structure endowed with an
assignment to $\ux$ and $\ua$, which were variables and now are treated as
symbols of $\tilde\Sigma$).  For the following, we need the following definition: 

\begin{definition}
A $\tilde\Sigma$-structure $\cM$ is called 
\emph{cyclic}\footnote{This notion comes from the universal algebra terminology.}  if
it is generated by a \emph{single} element $e\in E^{\cM}$ (called \emph{generator} of  $\cM$), where $E$ is a sort of index type (i.e. $e$ belongs to the interpretation of a
sort $E$ of index type). 
\end{definition}

The previous definition intuitively means that \emph{all} the elements of the cyclic structures are obtained from the generator by applying the function symbols of  $\tilde\Sigma$ to the generator. 

 Since $\Sigma$ is relational, one can show
that there are only finitely many cyclic $\tilde\Sigma$-structures
$\cC_1,\ldots,\cC_N$ up to isomorphism.  With a $\tilde\Sigma$-structure $\cM$
we associate the tuple of numbers
$k_1(\cM), \dots, k_N(\cM)\in \mathbb N\cup\{\infty\}$ counting the numbers of
elements generating (as singletons)  cyclic substructures isomorphic to
$\cC_1, \dots, \cC_N$, respectively. 

 Now, we show that, if the tuple
 associated with $\cM$ is component-wise bigger than the one associated with
 $\cN$, then $\cM$ satisfies all the local formulae satisfied by $\cN$.

\begin{lemma}\label{lem:localinv}
	Let $\cM, \cN$ be $\tilde \Sigma$-structures. If the inequalities
	\[
	k_1(\cM)\leq k_1(\cN), \dots, k_N(\cM)\leq k_N(\cN)
	\]
	hold, then all local formulae true in $\cM$ are also true in $\cN$.
\end{lemma}

\begin{proof}
	Notice that local formulae (viewed in $\tilde\Sigma$) are sentences, because they do not have free variable occurrences - the $\ua, \ux$ are now constant function symbols and individual constants, respectively.  The proof of the lemma is fairly obvious: notice that, once we assigned some $\alpha(e_i)$ in $\cM$ to the variable $e_i$, the truth of a formula like $\phi(e_i, \ux, \ua)$ under such an assignment depends only on the $\tilde \Sigma$-substructure generated by $\alpha(e_i)$, because  $\phi$ is quantifier-free and $e_i$ is the only $\tilde \Sigma$-variable occurring in it. 
	In fact, if a local state formula $ \exists e_1\cdots \exists e_k \left( \delta(e_1,\dots, e_k)  \land
	\bigwedge_{i=1}^k \phi_i(e_i,\ux,\ua)\right)$  is true in $\cM$, then there
	exist elements $\bar{e}_1,\cdots, \bar{e}_k$
	(in the interpretation of some index sorts),  each of which makes $\phi_i$ true. Hence, $\phi_i$ is also true  in the corresponding cyclic structure generated by $\bar{e}_i$. Since $k_1(\cM)\leq k_1(\cN), \dots, k_N(\cM)\leq k_N(\cN)$ hold, then also in $\cN$ there are at least as many elements in the interpretation of index sorts as there are in $\cM$ that validate all the $\phi_i$ . Thus, we get that the formula $\exists e_1\cdots \exists e_k \left( \delta(e_1,\dots, e_k)  \land
	\bigwedge_{i=1}^k \phi_i(e_i,\ux,\ua)\right)$ is true also in $\cN$, as wanted.
\end{proof}

The structure of the previous proof follows the schema of a proof from~\cite{MSCS20}, but there is a \emph{significant} difference: the decidability result  in \cite{MSCS20} is based on locality and applies to array-based systems whose FO-signatures do \emph{not} have relational symbols, whereas here our array-based systems \emph{do} have free relational symbols. This difference is reflected in the kind of FO-formulae that are involved.

Now we are ready to prove our first termination and decidability result. 

\vskip 2mm\noindent
\textbf{Theorem~\ref{thm:decid-local}}\emph{
$\textit{BReach}(\ABM,\goalab)$ always \emph{terminates} if \myi the \ABMAS is \async and \myii its transition  formula $\tau=\bigvee\hat\tau$ is a disjunction of \emph{local} transition formulae and \myiii $\goalab$ is local.}
\vskip 2mm
\begin{proof}
	Suppose, for the sake of contradiction, that the algorithm does not terminate. Then the fixpoint test of Line 2 fails infinitely often. Recalling that the $T$-equivalence of
	$B_n$  and of $\bigvee_{0\leq k<n} \phi_k$ is an invariant of the algorithm (here $\phi_n, B_n$ are the status of the variables 
	$\phi, B$ after $n$ execution of the main loop), this means that there are models
	\[
	\cM_0, \cM_1, \dots, \cM_j, \dots
	\]
	such that for all $j$, we have that  $\cM_j\models \phi_j$ and $\cM_j \not \models \phi_i$ (all $i<j$). But the $\phi_j$ are all local formulae, so considering the
	tuple of cardinals $k_1(\cM_j), \dots, k_N(\cM_j)$ and Lemma~\ref{lem:localinv}, we get a contradiction, in view of Dickson's Lemma. This is because, by Dickson's Lemma, $(\mathbb N \cup \{\infty\})^N$ is a wqo, so there exist $i$, $j$ such that $i<j$ and $k_1(\cM_i)\leq k_1(\cM_j), \dots, k_N(\cM_i)\leq k_N(\cM_j)$. Using Lemma~\ref{lem:localinv}, we get that $\phi_i$, which is local and true in $\cM_i$, is also true in $\cM_j$, which is a contradiction.
\end{proof}
\vskip 2mm

\end{document}